\documentclass[11pt,reqno]{amsart}
\usepackage{amsaddr}

\usepackage[margin=1in]{geometry}
\usepackage{fancyref}

\usepackage[utf8]{inputenc} 
\usepackage[T1]{fontenc}    
\usepackage{hyperref}       
\usepackage{url}            
\usepackage{booktabs}       
\usepackage{amsfonts}       
\usepackage{nicefrac}       
\usepackage{microtype}      
\usepackage{xcolor}         
\usepackage{amsmath}
\usepackage[capitalize]{cleveref}
\usepackage{bm}
\usepackage{bbm}
\usepackage{mathtools}
\usepackage{amsthm}

\usepackage[T1]{fontenc}

\usepackage{microtype}
\usepackage{graphicx}
\usepackage{booktabs} 
\usepackage{natbib}
\usepackage{algorithm}
\usepackage{algorithmic}
\usepackage{siunitx}
\usepackage{appendix}

\makeatletter
\def\paragraph{\@startsection{paragraph}{4}%
  \z@\z@{-\fontdimen2\font}%
  {\normalfont\bfseries}}
\makeatother

\usepackage{amsmath,amsfonts,bm}

\def\1{\bm{1}}

\DeclareMathAlphabet{\mathsfit}{\encodingdefault}{\sfdefault}{m}{sl}
\SetMathAlphabet{\mathsfit}{bold}{\encodingdefault}{\sfdefault}{bx}{n}

\usepackage{caption}
\usepackage{booktabs}
\usepackage{multirow}
\usepackage{enumitem}

\usepackage{subcaption}
\usepackage{hyperref}
\usepackage{url}
\usepackage{algorithm,algorithmic}
\usepackage{amsmath,amsthm}
\usepackage{amssymb,bm}
\usepackage{bbm}
\usepackage{makecell}
\newtheorem{theorem}{Theorem}

\newtheorem{lemma}{Lemma}
\newtheorem{proposition}{Proposition}
\newtheorem{remark}{Remark}

\newtheorem{corollary}{Corollary}
\usepackage{pgfplots}
\pgfplotsset{compat=1.10}
\usetikzlibrary{
        patterns,
        pgfplots.fillbetween,
    }

       \tikzset{
        hatch distance/.store in=\hatchdistance,
        hatch distance=10pt,
        hatch thickness/.store in=\hatchthickness,
        hatch thickness=2pt,
        hatch color/.store in=\hatchcolor,      
        hatch color=black,                      
    }
    \pgfdeclarepatternformonly[%
        \hatchdistance,%
        \hatchthickness,%
        \hatchcolor
            ]{flexible hatch}
    {\pgfqpoint{0pt}{0pt}}
    {\pgfqpoint{\hatchdistance}{\hatchdistance}}
    {\pgfpoint{\hatchdistance-1pt}{\hatchdistance-1pt}}%
    {
        \pgfsetlinewidth{\hatchthickness}
        \pgfpathmoveto{\pgfqpoint{0pt}{0pt}}
        \pgfpathlineto{\pgfqpoint{\hatchdistance}{\hatchdistance}}
        \pgfsetstrokecolor{\hatchcolor}
        \pgfusepath{stroke}
    }
\providecommand{\nor}[1]{\ensuremath{\left\lVert {#1} \right\rVert}}

\usepackage[makeroom]{cancel}

\hypersetup{
     colorlinks = true,
     linkcolor = red,
     anchorcolor = blue,
     citecolor = teal,
     filecolor = blue,
     urlcolor = black
     }
 \usepackage{tcolorbox}

\newtcolorbox{assbox}{colback=black!5!white,colframe=black!75!black}
  \newtcolorbox{thmbox}{colback=blue!5!white,colframe=black!75!black}

\newtheorem{assumption}{Assumption}

\title[Information Theoretic Guarantees For Policy Alignment In LLMs]{Information Theoretic Guarantees For Policy Alignment\\ In Large Language Models}

\author[Y. Mroueh]{Youssef Mroueh}
\address[Y.Mroueh]{IBM Research}

\begin{document}

\begin{abstract}

Policy alignment of large language models  refers to  constrained policy optimization, where the  policy is optimized to maximize a reward while staying close to a reference policy with respect to an $f$-divergence such as the $\mathsf{KL}$  divergence. The best of $n$ alignment policy selects a sample from the reference   policy that has the maximum reward among $n$ independent samples.  For both cases (policy alignment and best of $n$),  recent works showed empirically that the reward improvement of the aligned policy on the reference one scales like $\sqrt{\mathsf{KL}}$, with an explicit bound in $n$ on the $\mathsf{KL}$ for the best of $n$ policy. We show in this paper that the $\sqrt{\mathsf{KL}}$ information theoretic upper bound  holds if the reward under the reference policy has sub-gaussian tails. Moreover, we prove for the best of $n$ policy, that the $\mathsf{KL}$ upper bound can be obtained for any $f$-divergence via a reduction to exponential order statistics owing to the Rényi representation of order statistics, and a data processing inequality.  If additional information is known on the  tails of the aligned policy we show that tighter control on the reward improvement can be obtained via the Rényi divergence. Finally we demonstrate how these upper bounds transfer from proxy rewards to golden rewards which results in a decrease in the golden reward improvement due to overestimation and approximation errors of the proxy reward. 

\end{abstract}

\maketitle

\section{Introduction}
Aligning Large Language Models (LLMs) with human preferences  allows a tradeoff between maintaining the utility of the  pre-trained reference model and the alignment of the model with human values such as safety or other  socio-technical  considerations. Alignment is  becoming a crucial step in LLMs training pipeline,  especially as these models are leveraged in decision making  as well as becoming more and more accessible to the general public. Policy alignment starts by learning a reward model that predicts human preferences, these reward models are typically fine-tuned LLMs that are  trained  on pairwise human preference data \citep{NIPS2017_d5e2c0ad,stiennon2020learning,ouyang2022training,bai2022training}.  The reward is then optimized using \emph{training time alignment} i.e via  policy gradient  based reinforcement learning leading to the so called \emph{Reinforcemnent Learning from Human Feedback} (RLHF)  \citep{NIPS2017_d5e2c0ad}.  RLHF ensures that the reward is maximized while the policy $\pi$ stays close to the initial reference policy $\pi_{\mathrm{ref}}$ in the sense of the Kullback-Leibler divergence $\mathsf{KL} (\pi|| \pi_{\mathrm{ref}})$.  Other variants of these training time  alignment have been proposed via direct preference optimization  \citep{rafailov2024direct} \citep{zhao2023calibrating} \citep{ethayarajh2024kto}.
 Another  important  paradigm for optimizing the reward is \emph{ test time alignment}  via best of $n$ sampling from the reference policy and retaining the sample that maximizes the reward.  The resulting policy is known as the \emph{best of $n$ policy}. The best of $n$ policy is also used  in controlled decoding settings   \citep{yang-klein-2021-fudge,mudgal2023controlled}  and in fine-tuning LLMs to match the best of $n$ policy responses   \citep{touvron2023llama}.

\begin{figure}[ht!]
\vskip -0.1in
\centering
    \includegraphics[width=0.5\textwidth]{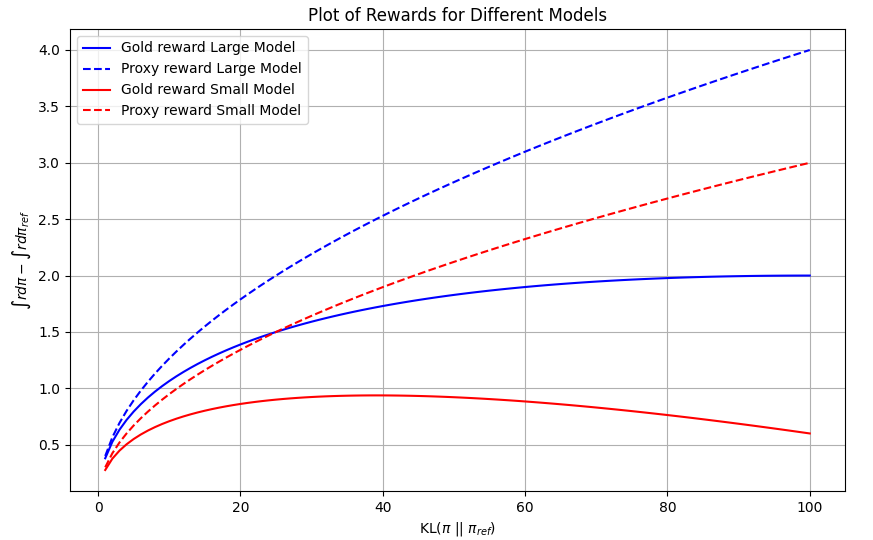}
    \caption{Qualitiative plot of centered rewards vs. KL of Proxy and Gold Rewards for both Best of $n$ and RL policies. (See Fig. 1 a) and b) in \citep{gao2023scaling} for scaling laws in policy alignment).}
    \label{fig:reward_vs_kl}
    \vskip -0.2in
\end{figure}
 \citep{gao2023scaling} and \citep{Goodhart} studied the scaling laws of reward models optimization in both the RL and the best of $n$ setups.  \citep{gao2023scaling} distinguished 
between ``golden reward'' that can be thought of as the golden human preference  and ``proxy reward'' which is trained to predict the golden reward.  For proxy rewards  \citep{gao2023scaling} found experimentally for both RL and  best of $n$ policies that the reward improvement on the reference policy scales as $\sqrt{\mathsf{KL} (\pi|| \pi_{\mathrm{ref}})}$. Similar observations for reward improvement scaling in RL were made in \citep{bai2022training}.  For golden rewards, \citep{gao2023scaling} showed for both RL and best of $n$ policies  that  LLMs that optimize the proxy reward suffer from  over-optimization in the sense that as the policy drifts from the reference policy, optimizing the proxy reward  results in  deterioration of the golden reward. This phenomena is referred to in  \citep{gao2023scaling} \citep{Goodhart} as Goodhart's law. A qualitative plot of scaling laws discovered in  \citep{gao2023scaling} is given in Figure \ref{fig:reward_vs_kl}. For the best of $n$ policy, most  works in this space assumed that  $\mathsf{KL} (\pi|| \pi_{\mathrm{ref}})= \log(n) - \frac{n-1}{n}$ \citep{stiennon2020learning,coste2024reward,nakano2021webgpt,go2024compositional,gao2023scaling}. Recently  \cite{beirami2024theoretical}
 showed that this is in fact an inequality under the assumption that the reward is one to one map (a bijection) and for finite alphabets. The main contribution of this papers are : 
\begin{enumerate}
\item We provide in Theorem \ref{Theo:/boundKL} in Section  \ref{sec:align} a new proof for the best of  $n$ policy inequality $\mathsf{KL} (\pi|| \pi_{\mathrm{ref}})\leq \log(n) - \frac{n-1}{n}$ and show  via a reduction to exponential random variables that it is a  consequence of the data processing inequality of the $\mathsf{KL}$ divergence. We extend this inequality beyond the setup of   \citep{beirami2024theoretical} of one to one rewards and finite alphabets to a more realistic setup of surjective rewards and beyond finite alphabets. We also give conditions under which the equality is met, and extend those inequalities to $f$-divergences and Rényi divergences. 
\item We show in Section \ref{sec:Trans} that the scaling laws  on policy improvement  versus $\mathsf{KL}$ of  \citep{gao2023scaling}  are  information theoretic upper bounds and  are consequences of transportation inequalities with the $\mathsf{KL}$ divergence under sub-gaussian tails of the reward under the reference policy. We discuss how the dependency on $\mathsf{KL}$ is driven only by the tails of the reward under the reference model, and cannot be improved by a better alignment algorithm and can only be improved if the tails of the reference rewards are fatter than sub-gaussian such as sub-gamma or sub-exponential tails. 
\item We study in Theorem \ref{theo:tailadapt} the tightness of these information theoretical upper bounds when the tails of the optimized policy are also known  via new transportation inequalities for the Rényi divergence $D_{\alpha}$ for in $\alpha \in (0,1)$, and show  that the upper bound $\sqrt{\mathsf{KL}}$ can not be met, echoing Goodhart's law of  \citep{gao2023scaling}.
\item We finally  study in Section \ref{sec:gold} the transfer of transportation inequalities from proxy rewards to golden rewards and prove that indeed the golden reward improvement is hindered by ``overestimation'' of the proxy reward as reported empirically in \citep{gao2023scaling}.
\end{enumerate}

\section{The Alignment Problem}\label{sec:align}

\subsection{RLHF: A Constrained Policy Optimization Problem }

Let $\mathcal{X}$ be the space of prompts and  $\mathcal{Y}$ be the space of responses $y \in \mathcal{Y}$ from a LLM conditioned on a prompt $x\in \mathcal{X}$. The reference LLM is  represented as policy $\pi_{\mathrm{ref}}(y|x)$, i.e as a conditional probability on $\mathcal{Y}$ given a prompt $x\in \mathcal{X}$. Let $\rho_{\mathcal{X}}$ be a distribution on prompts, and a $r$ a reward,  $r: \mathcal{X}\times \mathcal{Y} \to \mathbb{R}$, $r$ represents a safety or alignment objective that is desirable  to maximize. \\
Given a reference policy $\pi_{\mathrm{ref}}$, the goal of alignment is to find a policy $\pi^*$ that maximizes the reward $r$ and that it is still close to the original reference policy  for some positive $\Delta >0$: 
\begin{align}
\pi^*_{y|x} = \arg\max_{\pi_{y|x} }\mathbb{E}_{x\sim \rho_{\mathcal{X}}}\mathbb{E}_{y \sim \pi(.|x)}r(x,y) \text{ s.t } \int_{\mathcal{X}} \mathsf{KL}( \pi(y|x) || \pi_{\text{ref}}( y|x) )d\rho_{\mathcal{X}}(x) \leq \Delta,
\label{eq:prob}
\end{align}
where 
$\mathsf{KL}( \pi(y|x) || \pi_{\text{ref}}( y|x) ) = \mathbb{E}_{y\sim \pi{.|x}} \log\left(\frac{ \pi (y|x)}{ \pi_{\text{ref}}(y|x)}\right). $
With some abuse of notation, we write $ \pi(x,y)= \pi(y|x) \rho_{\mathcal{X}}(x)$ and $\pi_{\mathrm{ref}}(x,y)= \pi_{\mathrm{ref}}(y|x) \rho_{\mathcal{X}}(x)$.  Let $\mathcal{P}(\mathcal{X} \times \mathcal{Y})$ be joint probability  defined on $\mathcal{X}\times \mathcal{Y}$ that has $\rho_{\mathcal{X}}$ as marginal on $\mathcal{X}$. Hence we can write the alignment problem \eqref{eq:prob} in a more compact way as follows:
\vskip -0.2in
\begin{equation}
\sup_{\pi \in \mathcal{P}(\mathcal{X}\times \mathcal{Y} )} \int r d\pi  \text{ s.t }  \mathsf{KL} (\pi|| \pi_{\mathrm{ref}}) \leq \Delta. 
\label{eq:RL}
\end{equation}
For $\beta>0$, we can also write a penalized form of this constrained policy optimization problem as follows: $\sup_{\pi \in \mathcal{P}(\mathcal{X}\times \mathcal{Y} )} \int r d\pi -\frac{1}{\beta} \mathsf{KL} (\pi || \pi_{\mathrm{ref}} ).$
It is easy to see that the optimal policy of the penalized problem is given by:
\begin{equation}
  \pi_{\beta,r}(y|x)  = \frac{\exp(\beta r(x,y) ) \pi_{\mathrm{ref}}(y|x) }{\int \exp(\beta r(x,y) ) d\pi_{\mathrm{ref}}(y|x) }, \rho_{\mathcal{X}}  \text{almost surely}.  
  \label{eq:PenRL}
  \end{equation}
The constrained problem \eqref{eq:RL} has a similar solution (See for e.g \citep{yang2024asymptotics}):
  \begin{align}
  \pi_{\lambda_{\Delta},r}(y|x)  = \frac{\exp(\frac{ r(x,y)}{\lambda_{\Delta}} ) \pi_{\mathrm{ref}}(y|x) }{\int \exp(\frac{r(x,y)}{\lambda_{\Delta}} ) d\pi_{\mathrm{ref}}(y|x) }, \rho_{\mathcal{X}}  \text{almost surely},  
  \label{eq:rlpolicyoptimal}
  \end{align}
 where $\lambda_{\Delta} >0$ is a lagrangian that  satisfies $\int_{\mathcal{X}} \mathsf{KL}( \pi_{\lambda_{\Delta},r}(y|x) || \pi_{\text{ref}}( y|x) )d\rho_{\mathcal{X}}(x) = \Delta. $

\subsection{Best of $n$ Policy Alignment}

\noindent Let $X$ be the random variable associated with prompts such that Law($X)= \rho_{\mathcal{X}}$.  Let $Y$ be the random variable associated with the conditional response of $\pi_{\mathrm{ref}}$ given  $X$. Define the conditional reward of the reference policy  :
\[R(Y)| X: =r(X,Y) \text{ where Y } \sim \pi_{\text{ref}}(.|X) ,\]
we assume that $R(Y)| X$ admits a CDF denoted as $F_{R (Y) |X}$ and let $F^{-1}_{R (Y) |X}$ be its quantile:
$$F^{(-1)}_{R (Y) |X}(p)= \inf \{\eta : F_{R (Y) |X}(\eta) \geq p \} \text{ for } p \in [0,1].$$
 Let $Y_1\dots Y_n$ be independent samples from $\pi_{\text{ref}}(.|X)$. We define the best of $n$ reward as follows:
\begin{equation}
R^{(n)}(Y)|X  = \max_{i=1\dots n} R(Y_i )|X, 
\end{equation}
this the maximum of $n$ iid random variables with a common CDF $ F_{R (Y) |X}$.
The best of $n$ policy corresponds to
$Y^{(n)} |X := \arg\max_{i=1\dots n} r(X, Y_i) .$
We note $\pi^{(n)}_{r,\mathrm{ref}}(.|X)$  the law of $Y^{(n)}|X$.  $   \pi^{(n)}_{r,\mathrm{ref}} $ is referred to  as the best of $n$ alignment policy. We consider  two setups for the reward:
 \begin{assumption}
 We assume that the reward $r$ is a one to one map for a  fixed $x$, and admits an inverse $h_x: \mathbb{R}\to \mathcal{Y}$ such that $h_x(r(x,y))=y$.
 \label{assum:rewardBij}
 \end{assumption}
This assumption was considered in \citep{beirami2024theoretical}. Nevertheless this assumption is strong and not usually meet in practice, we weaken this assumption to the following:
 \begin{assumption}
We assume that there is a stochastic map $H_{X}$ such that $H_{X} (R_{Y|X})) \overset{d}{=} Y|X $ and  $H_{X} (R_{Y^{(n)}|X})) \overset{d}{=} Y^{(n)}|X $. 
 \label{assum:rewardStochasticmap}
 \end{assumption}
 Under  Assumption \ref{assum:rewardStochasticmap}, the reward can be surjective which is more realistic but we assume that there is a stochastic map that ensures invertibility not point-wise but on a distribution level.  Our assumption means that we have conditionnaly on $X$: $R|X  \to  Y|X$ form a markov chain i.e exists $A(Y|R,X)$ so that $P_{Y|X} =A(Y|R,X) P_{R|X},$ and 
 $P_{Y^{(n)}|X} =A(Y|R,X) P_{R^{(n)}|X}.$\\
 
\noindent \textbf{Best of $n$ Policy $\mathsf{KL}$ Guarantees: A reduction to Exponentials.} In what follows for random variables $Z,Z'$ with laws $p_{Z},p_{Z'}$ we write interchangeably: 
 $\mathsf{KL}(p_{Z}|| p_{Z'})= \mathsf{KL} (Z||Z').$
 \noindent Let us start by looking at
$ \mathsf{KL} \left[  R^{(n)}(Y) || R(Y) \Big|X  \right]$ the $\mathsf{KL}$ divergence between the conditional reward of the best of $n$ policy and that of the reference policy. 
 Let $E\sim Exp(1)$, the optimal transport map $F^{-1}_{R(Y)|X}\circ F_{E}$ from the exponential distribution  $E$ to $R(Y)|X$ (See for example Theorem 2.5 in \citep{santambrogio2015otam}: $E$ is atomless, but $R(Y)|X$ can be discrete valued) allows us to write:
 \begin{equation}
 R(Y)|X \stackrel{d}{=} F^{-1}_{R(Y)|X}\circ F_{E} (E),
 \label{eq:OTrep}
 \end{equation}
 where $ \stackrel{d}{=}$ means equality in distribution. 
 On the other hand, let $R^{(1)}(Y)|X \leq \dots \leq R^{(n)} (Y)|X$ be the order statistics of the rewards of  $n$ independent samples $Y_i,i=1\dots n$, $Y_i \sim \pi_{\mathrm{ref}}(.|X)$. The order statistics refer to sorting the random variable from the minimum (index $(1)$) to the maximum (index $(n)$).  Consider $n$ independent exponential $E_1,\dots E_n$, where $E_i\sim \exp(1)$, and their order statistics $E^{(1)}\leq E^{(2)}\leq \dots E^{(n)}$.  The Rényi representation of order statistics \citep{Rnyi1953OnTT},
 similar to the Optimal Transport  (OT) representation allows  us to express the distribution of the order statistics of the rewards in terms of the  order statistics of exponentials as follows:
 \begin{equation}
 \left( R^{(1)}(Y) | X ,\dots, R^{(n)}(Y)|X  \right) \overset{d}{=} \left(F^{-1}_{R(Y)|X}\circ F_{E} (E^{(1)}),\dots, F^{-1}_{R(Y)|X}\circ F_{E} (E^{(n)})   \right).
 \label{eq:renyiRep}
 \end{equation}
 The central idea in the Rényi representation is that the mapping $F^{-1}_{R(Y)|X}\circ F_{E} $ is monotonic and hence ordering preserving and by the OT representation each component is distributed as $R(Y)|X$.  See \citep{boucheron2012concentration} for more account on the Rényi representation of order statistics. 
 
 Hence using the OT representation in \eqref{eq:OTrep} and the Rényi representation of the maximum \eqref{eq:renyiRep}, we can reduce the $\mathsf{KL}$ between the rewards to a $\mathsf{KL}$ on functions of exponentials and their order statistics: 
 \begin{align}
  \mathsf{KL} \left[  R^{(n)}(Y) || R(Y) \Big|X  \right] & =   \mathsf{KL} \left( F^{-1}_{R(Y)|X}\circ F_{E} (E^{(n)}) \Big| \Big| F^{-1}_{R(Y)|X}\circ F_{E} (E)  \right) \nonumber \\
  &=  \mathsf{KL} (T_{X} (E^{(n) }) ||T_{X}( E)),
  \label{eq:OTandRenyi}
 \end{align}
where $T_{X}= F^{-1}_{R(Y)|X}\circ F_{E} = F^{-1}_{(r(X,.))_{\sharp}\pi_{\mathrm{ref}}(.|X)}  \circ F_{E}.$

Under Assumption \ref{assum:rewardBij} we can write samples from  the best of $n$ policy  as 
$ Y^{(n)}|X = h_{X}( R^{n}(Y)) | X $ 
 and from  the reference policy as
$ Y|X = h_{X}( R(Y)) | X .$ Hence we have by the data processing inequality (DPI) for the $\mathsf{KL}$ divergence  (See for e.g  \citep{infotheory}) under Assumption \ref{assum:rewardBij}:
 \begin{align}
 \mathsf{KL}(   \pi^{(n)}_{r,\mathrm{ref}} || \pi_{\text{ref}} |X ) &= \mathsf{KL}(  Y^{(n)} || Y | X) \nonumber\\
 &= \mathsf{KL} ( h_{X}( R^{n}(Y)) || h_{X}( R(Y))  | X ) \nonumber \\
 & =  \mathsf{KL}(R^{n}(Y) || R(Y) |X ) \text{ By Assumption \ref{assum:rewardBij} $h_X$ is one to one and DPI is an equality~}  \nonumber\\ 
 &= \mathsf{KL}(T_{X}(E^{(n)})||T_{X}( E)) \text{ Rényi and Optimal Transport Representations (Eq \eqref{eq:OTandRenyi}})
 \label{eq:chain}
 \end{align}
Recall that $T_{X}= F^{-1}_{R(Y)|X}\circ F_{E} $, $F_{E}$ is one to one. If the space $\mathcal{Y}$ is finite, $R(Y|X)$ has a discontinuous CDF hence  not strictly monotonic. It follows that its quantile $F^{-1}_{R(Y)|X}$ is not a one to one map and $T_{X}$ as a result is not a one to one map and hence we have by DPI (that is an inequality in this case since $T_{X}$ is not one to one):
\begin{equation}
\mathsf{KL}(T_{X}(E^{(n)})||T_{X}( E)) \leq \mathsf{KL}(E^{(n)}|| E)
\label{eq:Yfinite}
\end{equation}
 If the space $\mathcal{Y}$ is infinite and we assume  that $R(Y|X)$ is continuous and strictly monotonic then $F^{-1}_{R(Y)|X}$ is a one to one map, and as a result $T_{X}$ is a one to one map and the DPI is an equality in this case: 
 \begin{equation}
 \mathsf{KL}(T_{X}(E^{(n)})||T_{X}( E)) = \mathsf{KL}(E^{(n)}|| E)
 \label{eq:yinfinite}
\end{equation}

Hence under Assumption \ref{assum:rewardBij} and for $\mathcal{Y}$ finite combining \eqref{eq:chain}  and \eqref{eq:Yfinite} we have:
\begin{equation}
\mathsf{KL}(   \pi^{(n)}_{r,\mathrm{ref}} || \pi_{\text{ref}} |X ) \leq \mathsf{KL}(E^{(n)}|| E),
\end{equation}
and under Assumption \ref{assum:rewardBij} and for $\mathcal{Y}$ infinite and assuming $F_{R(Y)|X}$ is continuous  and strictly monotonic, combining \eqref{eq:chain} and \eqref{eq:yinfinite} we have:
\begin{equation}
\mathsf{KL}(   \pi^{(n)}_{r,\mathrm{ref}} || \pi_{\text{ref}} |X ) = \mathsf{KL}(E^{(n)}|| E).
\end{equation}
Under the more realistic Assumption \ref{assum:rewardStochasticmap} we can also apply the DPI on the stochastic map $H_{X}$, since DPI also holds for stochastic maps ( under our assumption $R| X \to Y|X$ see for example \citep{renyiPaper} Example 2)
\begin{align}
 \mathsf{KL}(   \pi^{(n)}_{r,\mathrm{ref}} || \pi_{\text{ref}} |X )& =  \mathsf{KL}(   H_{X}( R^{n}(Y)) || H_{X}( R(Y))  | X ))\nonumber\\
 &\leq \mathsf{KL}(R^{n}(Y) || R(Y) |X )= \mathsf{KL}(T_{X}(E^{(n)})||T_{X}( E)),
 \end{align}
 and hence under Assumption \ref{assum:rewardStochasticmap}  regardless whether $T
 _{X}$ is a one to one map or not, thus we have: $\mathsf{KL}(   \pi^{(n)}_{r,\mathrm{ref}} || \pi_{\text{ref}} |X ) \leq \mathsf{KL}(E^{(n)}|| E)$. The following Lemma gives a closed form expression for $\mathsf{KL}(E^{(n)} || E) $:
 \begin{lemma} [$\mathsf{KL}$ Between Exponential and Maximum of Exponentials] Let $E\sim \exp(1)$, and $E_1,\dots E_n$ be iid exponentials and $E^{(n)}$ their maximum, we have:
\begin{equation}
\mathsf{KL}(E^{(n)} || E) = \log(n)-\frac{n-1}{n}.
\end{equation} 
\label{lemma:KLexp}
\vskip -2in
\end{lemma}
\vskip -0.2in
Hence we conclude with the following result:
\vskip -0.4in
\begin{theorem} The best of n policy satisfies under  \hyperref[i]{(i)}
Assumption \ref{assum:rewardBij} (reward one to one) and for finite $\mathcal{Y}$ or  under  \hyperref[ii]{(ii)} Assumption \ref{assum:rewardStochasticmap} (existence of stochastic ``inverse'') :
\begin{equation}
 \mathsf{KL}(    \pi^{(n)}_{r,\mathrm{ref}} || \pi_{\mathrm{ref}} ) \leq\mathsf{KL}(E^{(n)} || E)  = \log(n) -\frac{n-1}{n}.
 \end{equation}
Under Assumption \ref{assum:rewardBij}, for infinite $\mathcal{Y}$  and assuming $F_{R(Y|X)}$ is continuous and strictly increasing for all $X$ we have:
\begin{equation}
 \mathsf{KL}(    \pi^{(n)}_{r,\mathrm{ref}} || \pi_{\mathrm{ref}} ) = \mathsf{KL}(E^{(n)} || E)  = \log(n) -\frac{n-1}{n}.
 \end{equation}
 \label{Theo:/boundKL}
\end{theorem}
\begin{proof} Combining Lemma \ref{lemma:KLexp}, the analysis above  and taking expectation on $X$ we obtain the result.
\end{proof}

 \cite{beirami2024theoretical} showed this result under condition (i) which is not a realistic setting and used the finiteness of $\mathcal{Y}$ to provide a direct proof.  Our analysis via chaining DPI and using OT and Rényi representations to reduce the problem to exponentials allows us  to extend  the result to a more realistic setup under condition (ii) i.e the existence of a stochastic ``inverse", without any assumption on $\mathcal{Y}$. Furthermore we unveil under which conditions the equality holds that was assumed to hold  in previous works \citep{stiennon2020learning} \citep{coste2024reward,nakano2021webgpt,go2024compositional} \citep{Goodhart} \citep{gao2023scaling}.\\  


Our approach of reduction to exponentials using Rényi representation of order statistics and  data processing inequalities  extends to bounding the $f$- divergence $\mathrm{D}_f(    \pi^{(n)}_{r,\mathrm{ref}} || \pi_{\mathrm{ref}} )$
as well as the $\alpha$ Rényi divergence. 
The Rényi divergence for $\alpha \in (0,1) \cup (1,\infty)$ is defined as follows:
 \[D_{\alpha}(P||Q) = \frac{1}{(\alpha-1)} \log \left( \int p^{\alpha}(x) q^{1-\alpha}(x) dx \right)\]
 the limit as $\alpha \to 1$ coincides with $\mathsf{KL}$, i.e: $D_{1}(P||Q)) = \mathsf{KL} (P||Q)$. These bounds are summarized in Table \ref{tab:f_divergence}. Full proofs and theorems are in the Appendix. \\

\begin{table}[htbp]
    \centering
    \begin{tabular}{|c|c|c|}
        \hline
        Divergence & $f(x)$ & Bound on $ \mathrm{D}_f(    \pi^{(n)}_{r,\mathrm{ref}} || \pi_{\mathrm{ref}} )$ \\
        \hline
        $\mathsf{KL}$ &$x\log(x)$ &  $\log(n) -\frac{n-1}{n}$ \\
        \hline
        Chi-squared & $(x-1)^2$ &$\frac{(n-1)^2}{2n-1}$ \\
        \hline
        Total Variation &$f(x)= \frac{1}{2} |x-1| $ &  $(\frac{1}{n})^{\frac{1}{n-1}}-  (\frac{1}{n})^{\frac{n}{n-1}}$ \\
        \hline
        Hellinger distance &$(1-\sqrt{x})^2$ & $ 2\frac{(1-\sqrt{n})^2}{n+1}$\\
        \hline
        Forward $\mathsf{KL}$ &$-\log(x)$ &$n-1 -\log(n)$ \\
        \hline
       $\alpha$ Rényi Divergence &NA &$\frac{1}{(\alpha-1)} \log \left( \frac{n^{\alpha}}{\alpha(n-1) +1}\right)$\\
               \hline
    \end{tabular}
        \caption{Best of $n$ policy $f$-Divergence and $\alpha$ Rényi Divergence Bounds.}
            \label{tab:f_divergence}
        \vskip -0.15in
\end{table}

\noindent \textbf{Best of $n$-Policy Dominance on the Reference Policy.}
The following proposition shows that the best of $n$ policy leads to an improved reward on average:
\begin{proposition}
$R^{(n)}$ dominates $R$ in the first order dominance that is $R^{(n)}$ dominates $R$ on all quantiles:
$ Q_{R^{(n)}} (t) \geq Q_{R}(t), \forall t \in [0,1].$ It follows that we have $\mathbb{E} R^{(n)} \geq \mathbb{E} R$. \\ 
\label{pro:Dominance}
\end{proposition}
 
\noindent \textbf{Best of $n$ Policy and  RL Policy} The following proposition  discusses the sub-optimality of the best of $n$ policy with respect to the alignment RL objective given in \eqref{eq:prob}: 
 \begin{proposition} Assume a bounded reward in $[-M,M]$.  For $\Delta>0$ and  $n =\exp(\Delta)$ the best of $n$ policy $   \pi^{(n)}_{r,\mathrm{ref}}$ and the $\Delta$ Constrained RL policy $ \pi_{\lambda_{\Delta},r}$ (given in \eqref{eq:rlpolicyoptimal})
 satisfy:
 \[ \mathsf{KL}(   \pi^{(n)}_{r,\mathrm{ref}}|| \pi_{\lambda_{\Delta},r} ) \leq \frac{\sqrt{2\pi} M(e^{\frac{2M}{\lambda_{\Delta}}}-1 )}{\lambda_{\Delta}} \exp(-\frac{\Delta}{2}).\]
 \label{pro:BNRL}
 \vskip-4in
 \end{proposition}
  \vskip-0.2in
  A similar asymptotic result appeared in \citep{yang2024asymptotics} for $\Delta \to \infty$, showing as $n\to \infty$, $\mathsf{KL}(   \pi^{(n)}_{r,\mathrm{ref}}|| \pi_{\lambda_{\Delta},r} ) \to 0$, we provide here a non asymptotic result for finite $n$ and finite $\Delta$. 
 \section{ Reward Improvement Guarantees Through Transportation Inequalities }
 \label{sec:Trans}
 \paragraph{Notations} Let $X$ be a real random variable. The logarithmic  moment generating function of $X$ is defined as follows for $\lambda \in \mathbb{R}$:
 $\psi_{X}(\lambda) = \log \mathbb{E}_{X} e^{\lambda (X- \mathbb{E} X)}.$
$X$ is said to be sub-Gaussian with variance $\sigma^2$ if : 
$\psi_{X}(\lambda) \leq \frac{\lambda^2 \sigma^2}{2} \text{for all } \lambda \in \mathbb{R}.$
 We denote $\mathsf{SubGauss}(\sigma^2)$ the set of sub-Gaussian random variables with variance $\sigma^2_{\mathrm{ref}}$. 
$X$ is said to be sub-Gamma on the right tail with variance factor $\sigma^2$ and a scale parameter $c>0$ if :
$ \psi_{X}(\lambda) \leq \frac{\lambda^2 \sigma^2}{2(1-c\lambda)} \text{for every } \lambda  \text{ such that } 0< \lambda <\frac{1}{c}.$
 We denote $\mathsf{SubGamma}(\sigma^2,c)$ the set of left and right tailed sub-Gamma random variables. Sub-gamma tails can be thought as an interpolation between sub-Gaussian and sub-exponential tails. \\

\paragraph{Scaling Laws in Alignment} It has been observed empirically    \citep{coste2024reward,nakano2021webgpt,go2024compositional,Goodhart,gao2023scaling}  that optimal RL policy $\pi_{\lambda_{\Delta},r}$  satisfy the following inequality for a constant $\sigma^2_{\mathrm{ref}}$
:$$ \mathbb{E}_{ \pi_{\lambda_{\Delta},r} } r- \mathbb{E}_{ \pi_{\mathrm{ref}} } r \leq  \sqrt{2\sigma^2 \mathsf{KL}(\pi_{\lambda_{\Delta},r} || \pi_{\mathrm{ref}}  )}.$$
A similar scaling for best of $n$ policy :
$$ \mathbb{E}_{   \pi^{(n)}_{r,\mathrm{ref}} } r- \mathbb{E}_{ \pi_{\mathrm{ref}} } r  \leq \sqrt{2\sigma^2 \left(\log n - \frac{n-1}{n}\right)},$$
and those bounds are oftentimes tight even when empirically estimated from samples.  
This hints that those bounds are information theoretic and independent of the alignment problem. Indeed if the reward was bounded, a simple application of Pinsker inequality gives rise to $\sqrt{\mathsf{KL}}$ scaling. Let $\mathsf{TV}$ be the total variation distance, we have:
$\mathsf{TV} (\pi, \pi_{\mathrm{ref}} ) = \frac{1}{2} \sup_{||r||_{\infty} \leq 1 } \mathbb{E}_{\pi} r -  \mathbb{E}_{\pi_{\mathrm{ref}}} r \leq \sqrt{\frac{1}{2} \mathsf{KL} (\pi|| \pi_{\mathrm{ref}})}.  $
Hence we can deduce that for bounded rewards $r$ with norm infinity $||r||_{\infty}$ that:
\[ \mathbb{E}_{\pi} r -  \mathbb{E}_{\pi_{\mathrm{ref}}} r \leq \sqrt{2||r||^2_{\infty} \mathsf{KL}  (\pi|| \pi_{\mathrm{ref}})}. \]
Nevertheless this boundedness assumption on the reward is not realistic, since most reward models are unbounded: quoting   \cite{lambert2024rewardbench} `` implemented by appending a linear layer to predict one logit or removing the final decoding layers and replacing them with a linear layer'' and hence the reward is unbounded by construction. We will show in what follows that those scalings laws are tied to the tails of the reward under the reference policy and are instances of transportation inequalities.

 \subsection{Transportation Inequalities with $\mathsf{KL}$ Divergence}
 For a policy  $\pi \in \mathcal {P}(\mathcal{Y})$ and for a reward function $r: \mathcal{Y} \to \mathbb{R}$ , we note $r_{\sharp}\pi $, the push-forward map of $\pi$ through $r$. The reader is referred to Appendix \ref{app:TranspKL} for background on transportation inequalities and how they are derived from the so-called Donsker-Varadhan variational representation of the $\mathsf{KL}$ divergence.  The following Proposition is an application of Lemma 4.14 in \citep{concentrationbook}): 
 
 \begin{proposition} [Transportation Inequalities] The following inequalities hold depending on the tails of $r_{\sharp}\pi_{\mathrm{ref}} $:
 \begin{enumerate}
 \item  Assume that $r_{\sharp}\pi_{\mathrm{ref}} \in \mathsf{SubGauss}(\sigma^2_{\mathrm{ref}})$.  For any  $\pi \in \mathcal{P}(\mathcal{Y})$ that  is absolutely continuous with respect to $\pi_{\mathrm{ref}}$, and such that  $ \mathsf{KL}(\pi|| \pi_{\mathrm{ref}}) < \infty $ then we have:
  \[ \left| \mathbb{E}_{ \pi} r- \mathbb{E}_{ \pi_{\mathrm{ref}} } r \right| \leq  \sqrt{2\sigma^2_{\mathrm{ref}} \mathsf{KL}(\pi || \pi_{\mathrm{ref}}  )}. \]
\item Assume that $r_{\sharp}\pi_{\mathrm{ref}} \in \mathsf{SubGamma}(\sigma^2_{\mathrm{ref}},c)$.  For any  $\pi \in \mathcal{P}(\mathcal{Y})$ that  is absolutely continuous with respect to $\pi_{\mathrm{ref}}$, and such that $ \mathsf{KL}(\pi|| \pi_{\mathrm{ref}}) < \infty $ then we have:
\[  \left| \mathbb{E}_{ \pi } r- \mathbb{E}_{ \pi_{\mathrm{ref}} } r  \right| \leq  \sqrt{2\sigma^2_{\mathrm{ref}} \mathsf{KL} (\pi || \pi_{\mathrm{ref}}) } + c  \mathsf{KL} (\pi || \pi_{\mathrm{ref}})  \]
 \end{enumerate} 
 \label{theo:subgaussTransp}
 \end{proposition}
 In particular we have the following Corollary:
 \begin{corollary}[Expected Reward Improvement] If $r_{\sharp}\pi_{\mathrm{ref}} \in \mathsf{SubGauss}(\sigma^2_{\mathrm{ref}})$ the following holds for the optimal RL policy $\pi_{\lambda_{\Delta,r}}$ and for the best of $n$ policy $  \pi^{(n)}_{r,\mathrm{ref}}$:
 \begin{enumerate}
 \item For the optimal RL policy  $\pi_{\lambda_{\Delta,r}}$ we have: 
  \[ 0 \leq \mathbb{E}_{\pi_{\lambda_{\Delta,r}}} r - \mathbb{E}_{\pi_{\mathrm{ref}}}r  \leq  \sqrt{2 \sigma^2_{\mathrm{ref}} \mathsf{KL} (\pi_{\lambda_{\Delta,r}} || \pi_{\mathrm{ref}}) } \leq \sqrt{2 \sigma^2_{\mathrm{ref}}\Delta}. \]
  \item For the Best of $n$ policy $  \pi^{(n)}_{r,\mathrm{ref}}$, under Assumption \ref{assum:rewardStochasticmap} we have:
  \[  0 \leq  \mathbb{E}_{   \pi^{(n)}_{r,\mathrm{ref}} } r- \mathbb{E}_{ \pi_{\mathrm{ref}} } r \leq  \sqrt{2\sigma^2_{\mathrm{ref}} \mathsf{KL}(   \pi^{(n)}_{r,\mathrm{ref}}  || \pi_{\mathrm{ref}}  )} \leq \sqrt{2\sigma^2_{\mathrm{ref}} \left(\log n - \frac{n-1}{n}\right)}.\]
  \vskip -0.2in
 \end{enumerate}
 \label{cor:bounds}
 \end{corollary}

  A similar statement holds  under sub-gamma tails of the reward of the reference model. We turn now to providing a bound in high probability on the empirical reward improvement of RL:
  
   \begin{remark} Item (1) in Corollary \ref{cor:bounds} shows that the $\sqrt{\sigma^2_{\mathrm{ref}}\mathsf{KL}}$ provides an upper bound on the reward improvement of the alignment under subgaussian tails of the reference reward. Under subgaussian tails of the reference, this information theoretic barrier can not be broken with a better algorithm. On way to improve on the $\sqrt{\mathsf{KL}}$ ceiling is by aiming at having a reference model with a reward that has subgamma tails to improve the upper limit to $\sqrt{ \sigma^2_{\mathrm{ref}}\mathsf{KL}} + c\mathsf{KL}$, or to subexponential tails to be linear in the $\mathsf{KL}$.  Item (2)
 can be seen as a refinement on the classical $\sqrt{2\sigma^2_{\mathrm{ref}} \log(n)}$ upper bound on the expectation of maximum of subgaussians see for e.g Corollary 2.6 in \citep{concentrationbook}. If in addition $r$ is positive and  for $X = r_{\sharp} \pi_{\mathrm{ref}} - \mathbb{E}_{\pi_{\mathrm{ref}}}r$ we have for $t>0$ , $\mathbb{P} (X > t ) \geq  \mathbb{P} (|g|> t)$, where $g\sim\mathcal{N}(0,\sigma^2_{\ell})$ (where $\sigma^2_{\ell}$ is a variance) , then  we have a matching lower bound for $  \pi^{(n)}_{r,\mathrm{ref}}$ that scales with $\sqrt{\sigma^2_{\ell}\log(n)}$ for sufficiently large $n$ (See \citep{kamath2015bounds}).
 \label{rem:remark}
  \end{remark}

The following Theorem gives high probability bounds for the excess reward when estimated from  empirical samples: 
 \begin{theorem} [High Probability Empirical Reward Improvement For RL] Assume  $r_{\sharp}\pi_{\mathrm{ref}} \in \mathsf{SubGauss}(\sigma^2_{\mathrm{ref}})$. Let $\beta>1$ and $t_0>0$. Let $\pi_{\beta,r}$ be the optimal policy of the penalized RL problem given in  Equation \eqref{eq:PenRL}. Let $R_{i,\beta}$ and $R_{i,\mathrm{ref}}, i=1\dots m$ be the rewards evaluated at $m$ samples from $\pi_{\beta,r}$ and $\pi_{\mathrm{ref}}$. Assume that the $\beta$-Rényi divergence $D_{\beta}(\pi_{\beta,r}  || \pi_{\mathrm{ref}} )$  and   $\mathsf{KL} (\pi_{\beta,r} || \pi_{\mathrm{ref}})$ are both finite.   The following inequality holds with probability at least $1-e^{-\frac{mt^2_0}{2\sigma^2_{\mathrm{ref}}}}-e^{-m (\beta-1) t_0}  $:
\begin{align*}
 \frac{1}{m} \sum_{i=1}^m  R_{i,\beta}     - \frac{1}{m} \sum_{i=1}^m R_{i,\mathrm{ref}} 
 &\leq \sqrt{2\sigma^2_{\mathrm{ref}} \mathsf{KL} (\pi_{\beta,r} || \pi_{\mathrm{ref}})} + \frac{D_{\beta}(\pi_{\beta,r}  || \pi_{\mathrm{ref}} )  -  \mathsf{KL} (\pi_{\beta,r} || \pi_{\mathrm{ref}}) }{\beta} +  2t_0 .
 \end{align*}
 \label{theo:highprob}
  \end{theorem}
   \vskip-0.1in
 Note that in Theorem \ref{theo:highprob}, we did not make any assumptions on the tails of $r_{\sharp}\pi_{\beta,r}$ and we see that this results in a biased concentration inequality with a non-negative bias $ \frac{D_{\beta}(\pi_{\beta,r}  || \pi_{\mathrm{ref}} )  -  \mathsf{KL} (\pi_{\beta,r} || \pi_{\mathrm{ref}}) }{\beta} \geq 0$.  For the best of $n$ policy, if the reward was positive and  has a folded normal distribution (absolute value of gaussians), \citep{boucheron2012concentration} provides concentration bounds, owing to the subgamma tails of the maximum of absolute value of Gaussians.

\subsection{Tail Adaptive Transportation Inequalities with the Rényi Divergence}
  
An important question on the tightness of the bounds  rises from the bounds in Corollary \ref{cor:bounds}. We answer this question by considering additional information on the tails of the reward under the policy $\pi$, and we obtain tail adaptive bounds  that are eventually tighter than the one in Corollary \ref{cor:bounds}. Our new bounds leverage a variational representation of the Rényi divergence that uses the logarithmic moment generating function of both measures at hand.\\

\paragraph{Preliminaries for the Rényi Divergence}
The Donsker-Varadahn representation of $\mathsf{KL}$ was crucial in deriving transportation inequalities.  In \cite{RenyiVform} the following variational form is given for the Rényi divergence in terms of the $\mathsf{KL}$ divergence, for all $\alpha \in \mathbb{R}$ 
\begin{equation}
(1-\alpha) D_{\alpha}(P || Q) = \inf_{R} \alpha \mathsf{KL}(R || P ) +(1-\alpha) \mathsf{KL}(R || Q)
\label{eq:primal}
\end{equation}
A similar variational form was rediscovered in \citep{anantharam2018variational}. Finally a  Donsker-Varadahn-Rényi representation of $D_{\alpha}$ was given in \citep{birrell2021variational}. For all $\alpha \in \mathbb{R}^{+}, \alpha \neq 0,1$ we have  : 
\begin{equation} 
\frac{1}{\alpha} D_{\alpha}(P || Q) = \sup_{h \in \mathcal{H}} \frac{1}{\alpha-1} \log \left( \mathbb{E}_{P}e^{(\alpha-1) h } \right)  - \frac{1}{\alpha} \log \left( \mathbb{E}_{Q} e^{\alpha h } \right), 
\label{eq:dual}
\end{equation}
where $ \mathcal{H} =\Big \{ h  \Big |  \int e^{(\alpha-1) h } dP < \infty,   \int e^{\alpha h } dQ <\infty \Big\}.$
\cite{birrell2021variational} presents a direct proof of this formulation without exploring its link to the representation given in \eqref{eq:primal}, we show in what follows an elementary proof via convex conjugacy, the duality relationship between equations \eqref{eq:primal}
 and \eqref{eq:dual}.
 
 \begin{theorem} For $0<\alpha<1$  Equations \eqref{eq:primal}
 and \eqref{eq:dual} are dual of one another. For $\alpha >1$ they are Toland Dual.
 \label{theo:Dual}
 \end{theorem}
 We collect in what follows elementary lemmas that will be instrumental to derive transportation inequalities in terms of the Rényi divergence. Proofs are given in the Appendix. 
 \begin{lemma} Let $\alpha \in (0,1)\cup (1,\infty)$, and define $\mathcal{H} =\{  h | e^{(\alpha-1)(h -\int hdP)} \in L^1(P) ,  e^{(\alpha)(h -\int h dQ)} \in L^1(Q) \}$.
We have for all $h \in \mathcal{H}$ and for $\alpha \in (0,1)\cup (1,\infty)$
\begin{equation*}
\int h dP - \int h dQ \leq \frac{1}{\alpha} D_{\alpha}(P || Q) -\frac{1}{\alpha-1} \log \left( \int e^{(\alpha-1) ( h -\int h dP) } dP \right) + \frac{1}{\alpha} \log \left( \int e^{\alpha ( h-   \int hdQ) } dQ \right)
\end{equation*}
\label{lem:varBRenyi}
\end{lemma}
 
\begin{lemma} The following limit holds for the Rényi divergence   $\lim_{\alpha \to 0} \frac{1}{\alpha} D_{\alpha}( P || Q) = \mathsf{KL} ( Q || P).$ 
\label{lem:limit0}
\end{lemma}


\paragraph{Transportation Inequalities with Rényi Divergence.} The following theorem shows that when considering the tails of $\pi$ we can obtain tighter upper bounds using the Rényi divergence that is more tail adaptive:
\begin{theorem} [Tail Adaptive Transportation Inequalities] Let $\alpha \in (0,1)$. Assume $r_{\sharp}\pi \in  \mathsf{SubGauss}(\sigma^2_{\pi}) $ and $ r_{\sharp}\pi_{\mathrm{ref}}\in \mathsf{SubGauss}(\sigma^2_{\mathrm{ref}}) $ then we have  for all $\alpha \in (0,1)$:
\begin{align}
\mathbb{E}_{ \pi} r- \mathbb{E}_{ \pi_{\mathrm{ref}} } r   \leq \sqrt{2((1-\alpha)\sigma^2_{\pi} + \alpha \sigma^2_{\mathrm{ref}})  \frac{D_{\alpha} (\pi  ||  \pi_{\mathrm{ref}})}{\alpha}}.
\label{eq:tailadapt}
  \end{align}  
\label{theo:tailadapt}
\end{theorem}

In particular if there exits $\alpha \in (0,1)$ such that $D_{\alpha}(\pi||\pi_{\mathrm{ref}}) \leq \frac{\alpha \sigma^2_{\mathrm{ref}}}{(1-\alpha)\sigma^2_{\pi} + \alpha \sigma^2_{\mathrm{ref}}} \mathsf{KL}(\pi || \pi_{\mathrm{ref}})$, then the tail adaptive upper bound given in Equation \eqref{eq:tailadapt}  is tighter than the one provided by the tails of $\pi_{\mathrm{ref}}$ only i.e  $\sqrt{\sigma^2_{\mathrm{ref}}   \mathsf{KL}(\pi || \pi_{\mathrm{ref}}) } $.  Note that this is possible because $D_{\alpha}$ is increasing in $\alpha \in (0,1)$ \citep{renyiPaper}, i.e  $D_{\alpha}(\pi||\pi_{\mathrm{ref}})\leq \mathsf{KL}(\pi || \pi_{\mathrm{ref}}) $, and $\frac{\alpha \sigma^2_{\mathrm{ref}}}{(1-\alpha)\sigma^2_{\pi} + \alpha \sigma^2_{\mathrm{ref}}} \leq 1$. Note that taking limits $\alpha \to 0 $ (applying Lemma \ref{lem:limit0}) and $\alpha \to 1$, and taking the minimum of the upper bounds we obtain:
\begin{align*}
\mathbb{E}_{ \pi} r- \mathbb{E}_{ \pi_{\mathrm{ref}} } r  \leq \sqrt{2 \min(\sigma^2_{\pi_{\mathrm{ref}}} \mathsf{KL}(\pi || \pi_{\mathrm{ref}}), \sigma^2_{\pi} \mathsf{KL}( \pi_{\mathrm{ref}} || \pi ) ) },
  \end{align*}    
this inequality can be also obtained by applying  Proposition \ref{theo:subgaussTransp} twice: on the tails of $\pi$ and $\pi_{\mathrm{ref}}$ respectively.

 Another important implication of Theorem \ref{theo:tailadapt}, other than tighter than $\mathsf{KL}$ upper bound, is that if we were to change  the RL alignment problem \eqref{eq:prob} to be constrained by $D_{\alpha}, \alpha \in (0,1)$ instead of   $\mathsf{KL}$, we may end up with a  smaller upper limit on the reward improvement.  This $D_{\alpha}$ constrained alignment may lead to a policy that under-performs when compared to a policy obtained with the $\mathsf{KL}$ constraint. This was indeed observed experimentally in \citep{wang2024beyond} that used constraints with $\alpha$- divergences for $\alpha \in (0,1) $ (that are related to Rényi divergences) and noticed a degradation in the reward improvement w.r.t the policy obtained using $\mathsf{KL}$ constraints. 
 
 \section{Transportation Inequality Transfer From Proxy to Golden Reward}\label{sec:gold}
 As we saw in the previous sections, the tightness of  $\sqrt{\mathsf{KL}(\pi || \pi_{\mathrm{ref}})}$ upper bound in alignment can be due to the tails of the reward of the aligned  policy $\pi$ (Theorem \ref{theo:tailadapt}) and to the concentration around the mean in finite sample size (Theorem \ref{theo:highprob}). Another important consideration is the mismatch between the golden reward $r^*$ that one desires to maximize that is expensive and difficult to obtain (for example human evaluation) and a proxy reward $r$ that approximates $r^*$.  The proxy reward $r$ is used instead of $r^*$ in RL and in best of $n$ policy. While we may know the tails of the reward $r$ of the reference and aligned model, we don't have access to this information on the golden reward $r^*$. We show in this section how to transfer transportation inequalities from $r$ to $r^*$ for RL and Best of $n$ policy.

\begin{proposition} [$r^*$ Transportation Inequality for RL Policy ] The following inequality holds:
\[  \mathbb{E}_{\pi_{\beta,r}} r^*  - \mathbb{E}_{\pi_{\mathrm{ref}}} r^*  \leq \mathbb{E}_{\pi_{\beta,r}} r   -  \mathbb{E}_{\pi_{\mathrm{ref}}} r  -  \frac{1}{\beta} \log \left(  \int e^{\beta(r-r^* - \left( \int r d\pi_{\mathrm{ref}} - \int r^* d\pi_{\mathrm{ref}} \right) } d\pi_{\beta,r^*} \right),\]
 Assume $r_{\sharp}\pi_{\mathrm{ref}} \in  \mathsf{SubGauss}(\sigma^2_{\mathrm{ref}})$, and there exists $\delta > 0$ such that:  $$\frac{1}{\beta} \log \left(  \int e^{\beta(r-r^* - \left( \int r d\pi_{\mathrm{ref}} - \int r^* d\pi_{\mathrm{ref}} \right) } d\pi_{\beta,r^*} \right) \geq \delta \mathsf{KL}(\pi_{\beta,r^*} || \pi_{\mathrm{ref}}), $$  then we have:
 \[ \mathbb{E}_{\pi_{\beta,r}} r^*  - \mathbb{E}_{\pi_{\mathrm{ref}}} r^*  \leq  \sqrt{2\sigma^2_{\mathrm{ref}} \mathsf{KL} (\pi_{\beta,r} || \pi_{\mathrm{ref}})}  - \delta  \mathsf{KL}(\pi_{\beta,r^*} || \pi_{\mathrm{ref}}). \]
 \label{pro:TransferRL}
\end{proposition}

Note that $\frac{1}{\beta} \log \left(  \int e^{\beta(r-r^* - \left( \int r d\pi_{\mathrm{ref}} - \int r^* d\pi_{\mathrm{ref}} \right) } d\pi_{\beta,r^*} \right)$ is interpreted here as an interpolation between the mean and the maximum of its argument on the support of $\pi_{\beta,r^*}$ (Proposition 9 in \citep{feydy2018interpolating}). Indeed as $\beta \to 0$, this boils down to the mean on $\int (r-r^*) d\pi_{\beta,r^*} - \left( \int r d\pi_{\mathrm{ref}} - \int r^* d\pi_{\mathrm{ref}} \right) $ and $\beta \to \infty$ this boils down to $\max_{\mathsf{supp}{\pi_{\beta,r^*}}} \{ r-r^* - \left( \int r d\pi_{\mathrm{ref}} - \int r^* d\pi_{\mathrm{ref}}\right) \}$. Our assumption means that  $r$ overestimates $r^*$ and the overestimation is accentuated as we drift from $\pi_{\mathrm{ref}}$ on which $r$ was learned. This assumption echoes findings in  \citep{gao2023scaling} that show that the transportation inequalities suffer from overestimation of proxy reward models of the golden reward (See Figure 8 in \citep{gao2023scaling}). 

Note that in Proposition \ref{pro:TransferRL}, we are evaluating the golden reward $r^*$ improvement when using the proxy reward optimal policy $\pi_{\beta,r}$. We see that the golden reward of the RL policy inherits the transportation inequality from the proxy one but the improvement of the reward is hindered by possible overestimation of the golden reward by the proxy model. This explains the dip in performance as measured by the golden reward depicted  in Figure \ref{fig:reward_vs_kl} and reported in  \citep{gao2023scaling}. 
%

\begin{proposition}[$r^*$ Transportation Inequality for Best of $n$ Policy] Let $\varepsilon >0$. Let $r$ be a surrogate reward such that $\nor{r-r^*}_{\infty} \leq \varepsilon$ and assume $r_{\sharp}\pi_{\mathrm{ref}} \in  \mathsf{SubGauss}(\sigma^2_{\mathrm{ref}})$  then the best of $n$ policy $\pi^{(n)}_{r, \mathrm{ref}}$ satisfies: 
\[ \mathbb{E}_{\pi^{(n)}_{r,\mathrm{ref}}} (r^*) - \mathbb{E}_{\pi_{\mathrm{ref}}} (r^*)  \leq \sqrt{2\sigma^2_{\mathrm{ref}} \left( \log(n) - \frac{n-1}{n}\right)} +2 \varepsilon \left( \left(\frac{1}{n}\right)^{\frac{1}{n-1}}-  \left(\frac{1}{n}\right)^{\frac{n}{n-1}}\right). \]
\label{pro:goohardtBestofn}
\end{proposition}
\vskip -0.2in
Transportation inequalities transfers for the  best of $n$ policy from $r$ to $r^*$ and pays only an additional error term $\nor{r-r^*}_{\infty} \mathsf{TV}( \pi^{(n)}_{r,\mathrm{ref}}  | \pi_{\mathrm{ref}}) $ , an upper bound of this total variation as a function of $n$ is given in Table \ref{tab:f_divergence}. As mentioned in remark \ref{rem:remark}, if we have lower bounds on the tail of the reference reward, then we also have a lower bound on the reward improvement that scales like $C \sqrt{ \sigma^2_{\ell} \log(n)} -2 \varepsilon \left( \left(\frac{1}{n}\right)^{\frac{1}{n-1}}-  \left(\frac{1}{n}\right)^{\frac{n}{n-1}}\right).$  This is in line with empirical findings in \citep{Goodhart} \citep{gao2023scaling} that showed that best of $n$ policy is resilient as the reward model $r$ gets closer to $r^*$.

\section{Conclusion}
We presented in this paper a comprehensive information  theoretical analysis of policy alignment using reward optimization with RL and best of $n$ sampling. We showed for best of $n$ a bound on $\mathsf{KL}$ under realistic assumptions on the reward. Our analysis showed that the alignment reward improvement, is intrinsically constrained by the tails of the reward under the reference policy and controlling the $\mathsf{KL}$ divergence results in an upper bound of the policy improvement. We showed that the $\mathsf{KL}$ bound may not be tight if the tails of the optimized policy satisfy a condition expressed via Rényi divergence. We also explained the deterioration of the golden reward via overestimation of the proxy reward.    

\bibliographystyle{abbrvnat}

\bibliography{iclr2024_conference.bib}
\newpage

\appendix

\section{Broader Impact and Limitations}
We believe this work explaining scaling laws for reward models and alignment will give practitioners insights regarding the limits of what is attainable via alignment. All assumptions under which our statements hold are given. We don't see any negative societal impact of our work. 
\section{Proofs For Best of $n$ Policy}

\subsection{Best of $n$ Policy $\mathsf{KL}$ Guarantees}

\begin{proof}[Proof of Lemma \ref{lemma:KLexp}]
\begin{align*}
\mathsf{KL}(E^{(n)} || E) &= \int_{0}^{+ \infty} f_{E^{(n)}}(x) \log\left( \frac{f_{E^{(n)}}(x) }{ f_{E}(x)}\right)dx
\end{align*}
We have $f_{E}(x)= e^{-x}\mathrm{1}_{x\geq 0}$. Note that the CDF of maximum of exponential $F_{E^{(n)}}(x)=(1-e^{-x}) \mathrm{1}_{x\geq 0},$ and hence $f_{E^{(n)}}(x)= n(1-e^{-x})^{n-1}e^{-x}\mathrm{1}_{x\geq 0} $. Hence we have:
\begin{align*}
\mathsf{KL}(E^{(n)} || E) &= \int_{0}^{+ \infty} n(1-e^{-x})^{n-1}e^{-x} \log\left( \frac{n(1-e^{-x})^{n-1}e^{-x}}{ e^{-x}}\right)dx\\
&= \int_{0}^{+ \infty} n(1-e^{-x})^{n-1}e^{-x} \log\left( n(1-e^{-x})^{n-1}\right)dx
\end{align*}
Let $u=1-e^{-x}$, we have $du =e^{-x}dx$. It follows that :
\begin{align*}
\mathsf{KL}(E^{(n)} || E) &= \int_{0}^{1} nu^{n-1} \log\left( nu^{n-1}\right)du\\
&= \int_{0}^{1} nu^{n-1} \left( \log (n)  + (n-1) \log(u)\right) du\\
& = \log(n) \int_{0}^{1}du^n + (n-1)\int_0^1 nu^{n-1}\log(u) du\\
&=  \log(n) + (n-1) \int_0^1 d (u^n \log u - \frac{ u^n}{n})\\
&= \log(n) -\frac{n-1}{n}.
\end{align*}
\end{proof}

\subsection{Best of n Policy f divergence and Rényi Divergence}

\paragraph{Best of $n$ Policy  $f$ divergence and Renyi  divergence Guarantees}

Given that our proof technique relies on DPI and Rényi representation, we show that similar results hold for any $f$-divergence and for the Rényi divergence:
\begin{equation}
D_{f}\left(P||Q\right) = \int q(x)f\left(\frac{p(x)}{q(x)}\right) dx,
\end{equation}
where $f$ is  convex  and $f(1)=0$. 
Hence we have by DPI for $f$-divergences:
\begin{assbox}
\begin{theorem} Under Assumption \ref{assum:rewardStochasticmap} the best of n policy satisfies for any f-divergence:
\begin{equation}
 \mathrm{D}_f(   \pi^{(n)}_{r,\mathrm{ref}} || \pi_{\mathrm{ref}} ) \leq \int_{0}^{1}  f\left( nu^{n-1}\right)du
 \end{equation}

 \label{theo:dfiv}
\end{theorem}
\end{assbox}
\begin{proof}[Proof of Theorem \ref{theo:dfiv}] 
 \begin{align}
 \mathrm{D}_{f}(  \pi^{(n)}_{r,\mathrm{ref}} || \pi_{\text{ref}} |X ) &= \mathrm{D}_f(  Y^{(n)} || Y | X) \nonumber\\
 &= \mathrm{D}_f ( H_{X}( R^{n}(Y)) || H_{X}( R(Y))  | X ) \nonumber \\
 & \leq  \mathrm{D}_f(R^{n}(Y) || R(Y) |X ) \text{ By the data processing inequality~}\\
 &= \mathrm{D}_f(T_{X}(E^{(n)})||T_{X}( E)) \text{ Renyi and Optimal Transport Representations \eqref{eq:OTandRenyi}}\nonumber\\
 &=  \mathrm{D}_f(E^{(n)}|| E) \text{ since $T_{X}$ is a monotonic bijection DPI is an equality}\\
 &= \int_{0}^{+ \infty} f_{E}(x) f\left( \frac{f_{E^{(n)}}(x) }{ f_{E}(x)}\right)dx\\
 &= \int_{0}^{\infty} (e^{-x}) f\left(n(1-e^{-x})^{n-1} \right)du\\
 &=\int_0^1 f(nu^{n-1})du.
 \label{eq:chainFDIV}
 \end{align}
 
  In particular we have the following bounds for common $f$ divergences:
\begin{itemize}
\item For $f(x) =x\log(x)$ we obtain the KL divergence and we have the result:
\[\int_0^1 n u^{n-1} \log(n u^{n-1}) du = \mathsf{KL}(E^{(n)} || E) = \log(n) -\frac{n-1}{n}. \]
\item For $f(x) = (x-1)^2$ we obtain the chi-squared divergence and we have:
$\int_0^1   \left( nu^{n-1}-1\right)^2du = \int_0^1 (n^2 u^{2(n-1)} -2n u^{n-1} +1) du =\frac{ n^2}{2n-1} u^{2n-1} -2 u^{n} +u |^1_0= \frac{n^2}{2n-1} -2+1 = \frac{ n^2-2n+1}{2n-1} = \frac{(n-1)^2}{2n-1} $.
\item For $f(x)= \frac{1}{2} |x-1|$, we obtain the total variation distance $\mathrm(TV)$ and we have:
$\frac{1}{2}\int_0^1   \left |  nu^{n-1}-1\right | du = \frac{1}{2} (\int_0^{u^*}   \left( 1- nu^{n-1}\right ) du  +(\int_{u^*}^{1}   \left(  nu^{n-1} -1\right ) du   ) =  (u^*- (u^*)^n), $where $n(u^*)^{(n-1)} =1$, i.e $u^*= (\frac{1}{n})^{\frac{1}{n-1}}$ . Hence the TV is $ (\frac{1}{n})^{\frac{1}{n-1}}-  (\frac{1}{n})^{\frac{n}{n-1}}$.
\item For $f(x) =(1-\sqrt{x})^2 $ we have  the hellinger distance:
$  \int_0^1   \left( \sqrt{n} u^{\frac{n-1}{2}}-1\right)^2du  = \int_0^1 (n u^{n-1} -2 \sqrt{n} u^{\frac{n-1}{2}} +1 ) du = u^n -2\sqrt{n} \frac{u^{ \frac{n+1}{2}}}{\frac{n+1}{2}} +u \Big|^1_0= 2(1- \frac{2\sqrt{n} }{n+1}) = 2\frac{(1-\sqrt{n})^2}{n+1} $
\item For $f(x)=-\log(x)$, we obtain the forward KL and we have : $\int_0^1 f(nu^{n-1}) du = n-1 -\log(n)$.
\end{itemize}
 \end{proof}

\paragraph{Guarantees with Rényi Divergence}

\noindent Turning now to the Rényi divergence for $\alpha \in (0,1) \cup (1,\infty)$:
 \[D_{\alpha}(P||Q) = \frac{1}{(\alpha-1)} \log \left( \int p^{\alpha}(x) q^{1-\alpha}(x) dx \right)\]
 the limit as $\alpha \to 1$ $D_{1}(P||Q)) = \mathsf{KL} (P||Q)$ . 

 \begin{assbox}
\begin{theorem} Under Assumption \ref{assum:rewardStochasticmap} the best of n policy satisfies:
\begin{equation}
 D_{\alpha}(   \pi^{(n)}_{r,\mathrm{ref}} || \pi_{\mathrm{ref}} ) \leq \frac{1}{(\alpha-1)} \log \left( \frac{n^{\alpha}}{\alpha(n-1) +1}\right)
 \end{equation}
 \label{theo:renyibestofn}
\end{theorem}
\end{assbox}
\vskip -0.8in
 \begin{proof}[Proof of Theorem \ref{theo:renyibestofn}]
 
 \noindent Applying DPI that holds also for the Rényi divergence twice from $Y,Y^{(n)}$ to $R,R^{(n)}$ and from $R,R^{(n)}$ to $E,E^{(n)}$ we obtain :
 \[  D_{\alpha}(  \pi^{(n)}_{r,\mathrm{ref}} || \pi_{\text{ref}} |X ) \leq D_{\alpha}(E^{(n)}|| E)\]
 
 \begin{align*}
 D_{\alpha}(E^{(n)}|| E) &=\frac{1}{(\alpha-1)} \log  \left(\int_0^{\infty} n^{\alpha} (1-e^{-x})^{\alpha(n-1)}e^{-\alpha x} e^{-x(1-\alpha)} dx\right)\\
 &= \frac{1}{(\alpha-1)} \log\left( \int_{0}^{+\infty} n^{\alpha} (1-e^{-x})^{\alpha(n-1)} e^{-x} dx\right)
 \end{align*}
 Let $u=1-e^{-x}$ we have $du=e^{-x} dx$
 \begin{align*}
  D_{\alpha}(E^{(n)}|| E)  &= \frac{1}{(\alpha-1)} \log\left(\int_0^1 n^{\alpha} u^{\alpha(n-1)}du \right)\\
  &= \frac{1}{(\alpha-1)} \left( \log n^{\alpha}  + \log \int_0^1 u^{\alpha(n-1)} du \right)\\
  &= \frac{1}{(\alpha-1)} \left( \log n^{\alpha}  + \log \frac{u^{\alpha(n-1) +1}}{\alpha(n-1)+1} \Big|^1_0\right)\\
  &=\frac{1}{(\alpha-1)} \log \left( \frac{n^{\alpha}}{\alpha(n-1) +1}\right)
 \end{align*}
\end{proof}

\paragraph{From Renyi to KL guarantees}
Let $ s_1(\alpha) =(\alpha-1)$ , and $s_2(\alpha) =\log \left( \frac{n^{\alpha}}{\alpha(n-1) +1}\right) $, we have  $D_{\alpha}(E^{(n)}|| E)= \frac{s_2(\alpha)}{s_1(\alpha)}$  , we have 
$ \mathsf{KL}(E^{(n)}|| E) =\lim_{\alpha \to 1}  D_{\alpha}(E^{(n)}|| E)= \lim_{\alpha \to 1} \frac{s_2(\alpha)}{s_{\alpha}} =\frac{0}{0}  $, hence applying L'Hôpital rule we have: $ \lim_{\alpha \to 1} \frac{s_2(\alpha)}{s_1(\alpha)}=  \lim_{\alpha \to 1} \frac{s'_2(\alpha)}{s'_1(\alpha)} = \lim_{\alpha \to 1} \frac{\log(n) - \frac{n-1}{\alpha(n-1) +1}}{1} = \log(n) -\frac{n-1}{n}$. Hence we recover the result for the $\mathsf{KL}$ divergence. 

\subsection{Best of $n$ Dominance}
\begin{proof}[Proof of Proposition \ref{pro:Dominance} ]
$F_{E^{(n)}}(x) = (F_{E}(x))^n \leq F_{E}(x), \forall x\geq 0$, which means also that $ F^{-1}_{E^{(n)}}(t) \geq F^{-1}_{E}(t) , \forall t \in [0,1]$, which means that $E^{(n)}$ dominates $E$ in the first stochastic order :  $E^{(n)}\underset{\text{FSD}}{\succcurlyeq} E $ , which means there exists a coupling between $E^{(n)}$ and $E$, $\pi\in \Pi(E^{(n)}, E)$, such that $E \geq e$, for all $(E,e) \sim \pi$.  On the other hand By Rényi and Monge map representations we have: $R^{(n)} =F^{-1}_{R}\circ F_{E} (E^{(n)})$ and $R =F^{-1}_{R}\circ F_{E}(E)$, given that $T= F^{-1}_{R}\circ F_{E} $ is non decreasing the same coupling $\pi$ guarantees that $T(E ) \geq  T(e)$, for all $(E,e) \sim \pi$ and Hence  $R^{(n)}\underset{\text{FSD}}{\succcurlyeq} R $.

\end{proof}

\begin{corollary} Best of n-polciy has higher expectation : 
\[ \mathbb{E} R^{(n)} \geq \mathbb{E} R, \]
and is a safer policy, let the Tail Value at Risk be:
\[\mathrm{TVAR}_{p} (X) = \frac{1}{p}\int_0^p Q_{R}(t) dt\]
We have $$ \mathrm{TVAR}_{p}(R^{n}) \geq \mathrm{TVAR}_{p}(R) , \forall p \in [0,1]  $$
\label{cor:dominance}
\end{corollary}\begin{proof} [Proof of Corollary \ref{cor:dominance}]First order dominance implies second order dominance (i.e by integrating quantiles). Expectation is obtained for $p=1$.
\end{proof}

\section{Best of $n$ and RL Policy }
\begin{proof} [Proof of Proposition \ref{pro:BNRL}] We fix here $\beta = \frac{1}{\lambda_{\Delta}}$
\begin{align*}
 \mathsf{KL}( \pi^{(n)}_{r,\mathrm{ref}}|| \pi_{\beta,r} ) &= \int \pi^{(n)}_{r,\mathrm{ref}}(y|x)  \log \left( \frac{\pi^{(n)}_{r,\mathrm{ref}} (y|x)}{\pi_{\beta,r}(y|x)}\right)
  =  \int \pi^{(n)}_{r,\mathrm{ref}}(y|x)  \log \left( \frac{\pi^{(n)}_{r,\mathrm{ref}} (y|x) }{ \pi_{\mathrm{ref}}(y|x) \frac{ e^{\beta r(x,y)} }{ Z_{\beta} (x)  } } \right)\\
 &= \mathsf{KL} (  \pi^{(n)}_{r,\mathrm{ref}} || \pi_{\mathrm{ref}}) + \log\left( \mathbb{E}_{ \pi_{\mathrm{ref}}} e^{\beta r}   \right)  - \beta \int r d\pi^{(n)}_{r,\mathrm{ref}}
\end{align*}
On the other hand by optimality of $\pi_{\beta,r}$ we have:
\[ \mathsf{KL}\left(\pi_{\beta,r} || \pi_{\mathrm{ref}}\right)  = \beta \int r d\pi_{\beta,r} - \log\left( \int e^{\beta r} d\pi_{\mathrm{ref}}\right)\\
 \]
and hence we have:
\begin{align*}
 \mathsf{KL}( \pi^{(n)}_{r,\mathrm{ref}}|| \pi_{\beta,r} )  &= \mathsf{KL} (  \pi^{(n)}_{r,\mathrm{ref}} || \pi_{\mathrm{ref}})  - \mathsf{KL}\left(\pi_{\beta,r} || \pi_{\mathrm{ref}}\right) + \beta\left(   \int r d\pi_{\beta,r} - \int r d\pi^{(n)}_{r,\mathrm{ref}} \right)
\end{align*}
We choose $n$ such that :
\[ \mathsf{KL} (  \pi^{(n)}_{r,\mathrm{ref}} || \pi_{\mathrm{ref}}) \leq \log(n) -\frac{n-1}{n} \leq \mathsf{KL}\left(\pi_{\beta,r} || \pi_{\mathrm{ref}}\right) = \Delta \]
and we conclude choosing $n=e^{\Delta}$ therefore for that choice of $n$ that:
\[ \mathsf{KL}( \pi^{(n)}_{r,\mathrm{ref}}|| \pi_{\beta,r} )  \leq \beta\left(   \int r d\pi_{\beta,r} - \int r d\pi^{(n)}_{r,\mathrm{ref}} \right) \]
On the other hand we have:
\begin{align*}
 \left|  \int r d\pi_{\beta,r} - \int r d\pi^{(n)}_{r,\mathrm{ref}} \right| & =  \left | \int r \exp(\beta r)  \frac{1}{Z_{\beta}}d\pi_{\mathrm{ref}} -   \int \max_i r(x_i)  d\pi_{\mathrm{ref}} (x_1)\dots d\pi_{\mathrm{ref}} (x_n)  \right| \\
  &=  \left| \int   \left(\frac{1}{n} \sum_{i=1}^n  \frac{r(x_i) \exp(\beta r(x_i))}{  Z_{\beta} }   - \max_i r(x_i)   \right) d\pi_{\mathrm{ref}} (x_1)\dots d\pi_{\mathrm{ref}} (x_n)  \right|\\
  & =  \left|  \int  \left(\frac{1}{n} \sum_{i=1}^n  \frac{r(x_i) \exp(\beta r(x_i))}{ \sum_{i=1}^n \exp(\beta r(x_i)) } \frac{\sum_{i=1}^n \exp(\beta r(x_i))}{Z_{\beta}}   - \max_i r(x_i)   \right) d\pi_{\mathrm{ref}} (x_1)\dots d\pi_{\mathrm{ref}} (x_n) \right|\\
  &\leq \int \left|  \max r(x_i) \left(  \frac{\frac{1}{n}\sum_{i=1}^n \exp(\beta r(x_i))}{Z_{\beta}} -1  \right) \right| d\pi_{\mathrm{ref}} (x_1)\dots d\pi_{\mathrm{ref}} (x_n) \\
  & \leq  \frac{M}{Z_{\beta}}  \mathbb{E} \left| \sum_{i=1}^n \exp(\beta r(x_i)) -Z_{\beta}\right |
\end{align*}
where we used the  following fact, followed by Jensen inequality : 
\[ \sum_{i=1}^n  \frac{r(x_i) \exp(\beta r(x_i))}{ \sum_{i=1}^n \exp(\beta r(x_i)) } \leq \max_i r(x_i).\]
Assume that the reward is bounded  hence we have by Hoeffding inequality :
\[ \mathbb{P} \left( \left | \frac{1}{n}\sum_{i=1}^n \exp(\beta r(x_i)) -   Z_{\beta} \right|  \geq  t \right) \leq  2 e^{-\frac{nt^2}{2 (\exp(\beta M)  - \exp(-\beta M) )^2}}  \]
Hence we have:
\[\mathbb{E} \left| \sum_{i=1}^n \exp(\beta r(x_i)) -Z_{\beta}\right | \leq 2 \sqrt{\frac{\pi}{2}} \frac{\exp(\beta M)  - \exp(-\beta M) }{\sqrt{n}}\]

\[ \mathsf{KL}( \pi^{(\exp(\Delta))}_{r,\mathrm{ref}}|| \pi_{\lambda_{\Delta},r} ) \leq \frac{M}{\lambda_{\Delta}Z_{1/\lambda_{\Delta}}}\sqrt{2\pi} (\exp(\beta M)  - \exp(-\beta M)) \sqrt{ \exp(-\Delta)}. \]
\end{proof}

\section{Transportation Inequalities and KL Divergence}
\subsection{Transportation Inequalities with KL} \label{app:TranspKL}

The following Lemma (Lemma 4.14 in \citep{concentrationbook}) uses the Donsker-Varadhan representation of the KL divergence to obtain bounds on the change of measure , and using the tails of $\pi_{\mathrm{ref}}$.

\begin{lemma} [Lemma 4.14 in \citep{concentrationbook}]  Let $\psi$  be a  convex and continuously differentiable function $\psi$ on a possibly unbounded interval $[0,b)$, and assume $\psi(0)=\psi'(0)=0$. Define for every $x\geq 0$, the convex conjugate $\psi^*(x) = \sup_{\lambda \in[0,b)} \lambda x -\psi(\lambda)$ , and let $\psi^{*-1}(t) =\inf \{x\geq 0: \psi^*(x)>t\}$. Then the following statements are equivalent:\\
(i) For $\lambda \in [0,b)$ \[\log \left( \int e^{\lambda(r-\int r dQ)} dQ \right) \leq \psi(\lambda),\] 
(ii) For any probability measure $P$ that is absolutely continuous with respect to $Q$ and such that $\mathsf{KL}(P || Q) < \infty$:
\[  \int r dP - \int r dQ \leq\psi^{*-1} (\mathsf{KL}(P||Q) ).\]
\end{lemma}

\begin{lemma} [ Inverse of the conjugate \citep{concentrationbook}]
\begin{enumerate}
\item If $Q \in \mathsf{SubGauss}(\sigma^2)$, we have for $t\geq 0$ $\psi^{*-1}(t)= \sqrt{2\sigma^2 t}.$
\item If $Q \in \mathsf{Subgamma}(\sigma^2,c)$, we have for $t\geq 0$ $\psi^{*-1}(t)= \sqrt{2\sigma^2 t} +c t$.
\end{enumerate}
\end{lemma}

We give here a direct proof for the subgaussian case:
\begin{proof}

By the Donsker Varadhan representation of the $\mathsf{KL}$ we have:
\[ \mathsf{KL}( P || Q) = \sup_{h} \int h dP - \log \left( \int e^ h dQ \right)  \]
\noindent Fix $x$  and $M>0$  and define  for $0< \lambda < M$ 
 $$h_{\lambda}(y)= \lambda\left( r(x,y ) - \mathbb{E}_{\pi_{\mathrm{ref}}(y|x)  }r(x,y)\right)$$
We omit in what follows $x$ and $y$, but the reader can assume from here on that $\pi$ and $\pi_{\mathrm{ref}}$ are conditioned on $x$. 
Note that $R_{\mathrm{ref}}|x = (r(x,.))_{\sharp} \pi_{\mathrm{ref}} (.|x)$ and we assume $R_{\mathrm{ref}}|x$ subgaussian.
   Note that 
$$\mathbb{E}_{\mathrm{\pi}_{\mathrm{ref}}} e^{h_{\lambda}} =  \mathbb{E}_{\pi_{\mathrm{ref}}|x} e^{\lambda ( r - \mathbb{E}_{\pi_{\mathrm{ref}} |x}r )} = M_{R_{\mathrm{ref}} |x} (\lambda),  $$
where $M_{R_{\mathrm{ref}} |x} $ the moment generating function of the reward under the reference policy.  $R_{\mathrm{ref}} |x$ is subgaussian we have for all $\lambda \in \mathbb{R}$: 
\[ \mathbb{E}_{\mathrm{\pi}_{\mathrm{ref}} |x} e^{h_{\lambda}} \leq e^{\frac{\lambda^2 \sigma^2}{2 }} \leq   e^{\frac{M^2 \sigma^2}{2 }}<\infty\]

Hence $h_{\lambda} \in \mathcal{H}$ and we have for all $\pi <\!\!<\pi_{\mathrm{ref}}$ and for all $0<M<\infty$ and $0<\lambda<M$:

\[ \lambda  \mathbb{E}_{\pi|x } ( r - \mathbb{E}_{\pi_{\mathrm{ref}}|x }r)   \leq   \mathsf{KL} (\pi || \pi_{\mathrm{ref}} |x) +  \log \left( \mathbb{E}_{\pi_{\mathrm{ref}} |x} e^{\lambda ( r - \mathbb{E}_{\pi_{\mathrm{ref}} |x}r )}\right)  \]
or equivalently:
\[  \mathbb{E}_{\pi |x}  r  - \mathbb{E}_{\pi_{\mathrm{ref}} |x}r  \leq \frac{1}{\lambda} \mathsf{KL} (\pi || \pi_{\mathrm{ref}} |x) + \frac{1}{\lambda} \log \left( \mathbb{E}_{\pi_{\mathrm{ref}}|x} e^{\lambda ( r - \mathbb{E}_{\pi_{\mathrm{ref}} |x}r )}\right) \]
Finally we have for $\pi<\!\!<\pi_{\mathrm{ref}}$ for all $0<\lambda<M$: 
\begin{equation}
 \mathbb{E}_{\pi |x}  r  - \mathbb{E}_{\pi_{\mathrm{ref}} |x}r  \leq \frac{1}{\lambda} \mathsf{KL} (\pi || \pi_{\mathrm{ref}} |x) + \frac{1}{\lambda} \log \left( M_{R_{\mathrm{ref}} |x} (\lambda) \right)  
 \label{eq:masterEq}
 \end{equation}
 Being a subgaussian, the MGF of $R_{\mathrm{ref}}|x$  is bounded as follows: 
\[ \log \left( M_{R_{\mathrm{ref}} |x} (\lambda) \right) \leq \frac{\lambda^2 \sigma^2}{2}.\] 
Hence we have for  :
\[  \mathbb{E}_{\pi |x}  r  - \mathbb{E}_{\pi_{\mathrm{ref}} |x}r  \leq \frac{1}{\lambda} \mathsf{KL} (\pi || \pi_{\mathrm{ref}} |x) + \frac{\lambda \sigma^2 }{2}  \]
Integrating over $x$ we obtain for all $\pi<\!\!<\pi_{\mathrm{ref}} $ and all $0< \lambda <M$:
\[  \mathbb{E}_{\pi }  r  - \mathbb{E}_{\pi_{\mathrm{ref}}}r  \leq \frac{1}{\lambda} \mathsf{KL} (\pi || \pi_{\mathrm{ref}}) + \frac{\lambda \sigma^2 }{2}  \]
Define : 
\[\delta(\lambda) = \frac{1}{\lambda} \mathsf{KL} (\pi ||\pi_{\mathrm{ref}}) + \frac{\lambda \sigma^2 }{2} \]
minimizing  the upper bound $\delta(\lambda)$ for $\lambda \in (0,M]$, taking derivative $\delta'(\lambda)=-\frac{\mathsf{KL} (\pi || \pi_{\mathrm{ref}}) }{\lambda^2} +\frac{\sigma^2}{2}=0$ gives $\lambda^* = \sqrt{\frac{2 \mathsf{KL}(\pi || \pi_{\mathrm{ref}})}{\sigma^2 }}$. Taking $M=2\lambda^*,$ $\lambda^*$ is the minimizer.  Putting this in the bound we have finally for all rewards $r$ for all $\pi$: 
\begin{equation}
 \mathbb{E}_{\pi }  r  - \mathbb{E}_{\pi_{\mathrm{ref}}}r  \leq \sqrt{2 \sigma^2  \mathsf{KL} (\pi || \pi_{\mathrm{ref}})}.  
 \end{equation}

\end{proof}

\begin{proof}[Proof of Corollary \ref{cor:bounds}]
(i) This follows from optimality of $\pi_{\lambda_{\Delta}}$ and applying the transportation inequality for gaussian tail.

(ii) This follows from applying  Corollary \ref{cor:dominance} (best of $n$ policy has larger mean ) and   \ref {Theo:/boundKL} for bounding the $\mathsf{KL}$.

%
 
\end{proof}

\begin{proof} [Proof of Theorem \ref{theo:highprob}]
For  the penalized RL we have by optimality: 
\begin{align*}
  \int r d\pi_{\beta, r} -\frac{1}{\beta} \mathsf{KL} (\pi_{\beta,r} || \pi_{\mathrm{ref}} )& = \frac{1}{\beta} \log\left( \int e^{\beta r } d\pi_{\mathrm{ref}}\right) \\
 &=  \frac{1}{\beta} \log\left( \int e^{\beta ( r - \int r d\pi_{\mathrm{ref}}) } d\pi_{\mathrm{ref}}\right) + \int r d\pi_{\mathrm{ref}} 
\end{align*}
It follows that : 
\begin{align} 
 \frac{1}{\beta} \log\left( \int e^{\beta ( r - \int r d\pi_{\mathrm{ref}}) } d\pi_{\mathrm{ref}}\right) & = \int r d\pi_{\beta,r} - \int r d\pi_{\mathrm{ref}} -\frac{1}{\beta} \mathsf{KL} (\pi_{\beta,r} || \pi_{\mathrm{ref}} )
 \label{eq:opt}
 \end{align}
On the other hand by the variational representation of the Rényi divergence we have:
\begin{align}
\int r d\pi_{\beta,r}  - \int r d \pi_{\mathrm{ref}}  &\leq \frac{D_{\beta}(\pi_{\beta,r}  || \pi_{\mathrm{ref}} ) }{\beta} -\frac{1}{\beta-1} \log \left( \int e^{(\beta-1) ( r -\int r d\pi_{\beta,r} ) } d\pi_{\beta,r} \right) \nonumber\\
&+ \frac{1}{\beta} \log\left( \int e^{\beta ( r - \int r d\pi_{\mathrm{ref}}) } d\pi_{\mathrm{ref}}\right)
\label{eq:renyivarbeta}
\end{align}
Summing Equations \eqref{eq:opt} and \eqref{eq:renyivarbeta} we obtain a bound on the moment generating function at $\beta$ of $ r_{\sharp} \pi_{\beta,r}$ (this is not a uniform bound , it holds only for $\beta$):
\begin{align}
\frac{1}{\beta-1} \log \left( \int e^{(\beta-1) ( r -\int r d\pi_{\beta,r} ) } d\pi_{\beta,r} \right) & \leq \frac{D_{\beta}(\pi_{\beta,r}  || \pi_{\mathrm{ref}} )  -  \mathsf{KL} (\pi_{\beta,r} || \pi_{\mathrm{ref}}) }{\beta}.
\end{align}

Let us assume $\beta >1$ we have therefore the following bound on the logarithmic moment generation function at $\beta-1$
\[ \psi_{r_{\sharp} \pi_{\beta,r}} (\beta-1) \leq \frac{\beta-1}{\beta} \left( D_{\beta}(\pi_{\beta,r}  || \pi_{\mathrm{ref}} )  -  \mathsf{KL} (\pi_{\beta,r} || \pi_{\mathrm{ref}})  \right)  \]
Let $R_{i,\beta} =r_{\sharp} \pi_{\beta,r},i=1\dots m  $ , the reward evaluation of $m$ independent samples of  $\pi_{\beta,r}$ we have: 
\begin{align}
\mathbb{P} \Big \{ \sum_{i=1}^m  (R_{i,\beta} - \int r d\pi_{\beta,r})  > m t \Big \} & =  \mathbb{P} (  e^{\sum_{i=1}^m (\beta-1)  (R_{i,\beta} - \int r d\pi_{\beta,r})}  > e^{m(\beta-1) t} )  \nonumber\\
&\leq e^{- (\beta-1) mt} e^{m \psi_{R_{\beta}}(\beta-1)  } \nonumber\\
& \leq  e^{- (\beta-1) mt}  e^{m  \frac{\beta-1}{\beta} \left( D_{\beta}(\pi_{\beta,r}  || \pi_{\mathrm{ref}} )  -  \mathsf{KL} (\pi_{\beta,r} || \pi_{\mathrm{ref}})  \right) }\nonumber\\
& \leq e^{-m (\beta-1) \left( t  -  \frac{D_{\beta}(\pi_{\beta,r}  || \pi_{\mathrm{ref}} )  -  \mathsf{KL} (\pi_{\beta,r} || \pi_{\mathrm{ref}}) }{\beta} \right )}
\end{align}
Let $t_0>0$, hence we have for $\beta >1$: 

\[ \mathbb{P} \Big \{ \frac{1}{m} \sum_{i=1}^m  R_{i,\beta} >  \int r d\pi_{\beta,r} + t_0 +  \frac{D_{\beta}(\pi_{\beta,r}  || \pi_{\mathrm{ref}} )  -  \mathsf{KL} (\pi_{\beta,r} || \pi_{\mathrm{ref}}) }{\beta} \Big     \}  \leq e^{-m (\beta-1) t_0}\]
Now turning to $R_{\mathrm{ref}}= r_{\sharp} \pi_{\mathrm{ref}} $, since $R_{\mathrm{ref}} \in \mathsf{SubGauss}(\sigma^2_{\mathrm{ref}}) $ we have for every $t_0>0$ :
\[ \mathbb{P} \Big \{ - \frac{1}{m} \sum_{i=1}^m R_{i,\mathrm{ref}}   > - \int r d\pi_{\mathrm{ref}} +  t_0 \Big \} \leq e^{-\frac{mt^2_0}{2\sigma^2_{\mathrm{ref}}}}\]
Hence we have with probability at least $1-e^{-\frac{mt^2_0}{2\sigma^2_{\mathrm{ref}}}}-e^{-m (\beta-1) t_0}  $:
\begin{align*}
 \frac{1}{m} \sum_{i=1}^m  R_{i,\beta}     - \frac{1}{m} \sum_{i=1}^m R_{i,\mathrm{ref}} & \leq  \int r d\pi_{\beta,r} -  \int r d\pi_{\mathrm{ref}} + 2t_0 + \frac{D_{\beta}(\pi_{\beta,r}  || \pi_{\mathrm{ref}} )  -  \mathsf{KL} (\pi_{\beta,r} || \pi_{\mathrm{ref}}) }{\beta} \\
 &\leq \sqrt{2\sigma^2_{\mathrm{ref}} \mathsf{KL} (\pi || \pi_{\mathrm{ref}})} +  2t_0 + \frac{D_{\beta}(\pi_{\beta,r}  || \pi_{\mathrm{ref}} )  -  \mathsf{KL} (\pi_{\beta,r} || \pi_{\mathrm{ref}}) }{\beta}.
 \end{align*}
 \end{proof}

\section{Proofs for Transportation Inequalities and Rényi Divergence}

\begin{proposition}[Fenchel Conjugate Propreties] Let $F$ and $G$ be convex functions on a space $E$ and $F^*$, $G^*$ be their convex conjugates defined on $E^*$. We have:
\begin{enumerate}

\item  Let $F_{\gamma}(x) = \gamma F\left(  x \right) $ we have:
\begin{equation}
F^*_{\gamma}(p) = \gamma F^*\left(\frac{p}{\gamma}\right)
\end{equation}
\item Duality:
\begin{equation}
\min_{x\in E} F(x) + G(x) = \max_{p \in E^*} -F^*(-p) - G^*(p)
\end{equation}
\item Toland Duality:
\begin{equation}
\min_{x\in E} F(x) - G(x)  = \min_{p} G^*(p) -F^*(p)
\end{equation}
\end{enumerate}
\label{pro:fenchelConj}
\end{proposition}

\begin{proof}[Proof of Theorem \ref{theo:Dual}]
 Let  $\gamma>0$ , let $F_{P,\gamma}(R)  = \gamma\mathsf{KL}(  R || P )$,  the Fenchel conjugate of $F_{P,1}(.)$ is defined  for $h$ bounded and measurable function as follows   $F^*_{P,1}(h ) = \log \mathbb{E}_{P} e^{h}.$ It follows by 1) in Proposition \ref{pro:fenchelConj}  that : $F^*_{P,\gamma}(h ) = \gamma F^*_{P,1}(\frac{h}{\gamma}) =  \gamma \log \mathbb{E}_{P} e^{\frac{h}{\gamma}} $.\\
\noindent  \underline{For $0<\alpha<1$:}
  The objective function in \eqref{eq:primal} is the sum of convex functions: $F_{P,\alpha}(R) + F_{Q,1-\alpha}(R) $, by (2) in Proposition \ref{pro:fenchelConj}, we have by duality:
  \begin{align*} 
  (1-\alpha) D_{\alpha}(P || Q) & =  \inf _{R } F_{P,\alpha}(R) + F_{Q,1-\alpha}(R)\\
  &= \sup_{h \in \mathcal{H}}   - F^*_{P,\alpha}(-h )  - F^*_{Q,1-\alpha}(h) \\
  &= \sup_{h \in \mathcal{H}}   - \alpha \log \mathbb{E}_{P} e^{-\frac{h}{\alpha}} - (1-\alpha) \log \mathbb{E}_{Q} e^{\frac{h}{1-\alpha}}
   \end{align*}
 Replacing $h$ by $(1-\alpha)(\alpha) h$ does not change the value of the sup  and hence we obtain:
 \begin{align*}
  (1-\alpha) D_{\alpha}(P || Q) &=  \sup_{h \in \mathcal{H}}   - \alpha \log \mathbb{E}_{P} e^{-\frac{(1-\alpha)(\alpha) h}{\alpha}} - (1-\alpha) \log \mathbb{E}_{Q} e^{\frac{(1-\alpha)(\alpha) h}{1-\alpha}}\\
  &=  \sup_{h \in \mathcal{H}}   - \alpha \log \mathbb{E}_{P} e^{- (1-\alpha)h}  -  (1-\alpha) \log \mathbb{E}_{Q} e^{\alpha h}.
 \end{align*}
 dividing by $\frac{1}{\alpha(1-\alpha)}$ both sides we obtain for $0<\alpha<1$:
 
 \[ \frac{1}{\alpha} D_{\alpha} ( P || Q) =  \sup_{h \in \mathcal{H}}    - \frac{1}{1-\alpha} \log \mathbb{E}_{P} e^{- (1-\alpha)h}  -  \frac{1}{\alpha} \log \mathbb{E}_{Q} e^{\alpha h}   \]
 
\noindent \underline{For $\alpha >1$: }  The objective function in \eqref{eq:primal} is the difference of convex functions: $F_{P,\alpha}(R) - F_{Q,\alpha-1}(R) $, by Toland Duality (3) in Proposition \ref{pro:fenchelConj} we have:
 \begin{align*} 
  (1-\alpha) D_{\alpha}(P || Q) & =  \inf _{R } F_{P,\alpha}(R)  - F_{Q,\alpha -1}(R)\\
  & = \inf_{h \in \mathcal{H}} F^*_{Q,\alpha -1}(h)  - F^*_{P,\alpha}(h)\\ 
  & = \inf_{h \in \mathcal{H}} (\alpha-1) \log\mathbb{E}_Qe^{\frac{h}{(\alpha-1)}} - \alpha \log \mathbb{E}_{P} e^{\frac{h}{\alpha}}
  \end{align*}
  The inf does not change when we replace $h$ by $\alpha(\alpha-1)h$, hence we have:
  \begin{align*} 
  (\alpha-1) D_{\alpha}(P || Q)  &= - \inf_{h \in \mathcal{H}} (\alpha-1) \log\mathbb{E}_Qe^{\frac{\alpha(\alpha-1)h }{(\alpha-1)}} - \alpha \log \mathbb{E}_{P} e^{\frac{\alpha(\alpha-1)h}{\alpha}}\\
  & = \sup_{h \in \mathcal{H}}  \alpha \log \mathbb{E}_{P} e^{(\alpha-1)h} - (\alpha-1) \log\mathbb{E}_Qe^{\alpha h }
  \end{align*}
dividing both sides  by $\frac{1}{\alpha(\alpha-1)}$ we obtain for $\alpha >1$:
\[ \frac{1}{\alpha} D_{\alpha} ( P || Q) =   \sup_{h \in \mathcal{H}}  \frac{1}{\alpha-1} \log \mathbb{E}_{P} e^{(\alpha-1)h} - \frac{1}{\alpha} \log\mathbb{E}_Qe^{\alpha h }.  \]
 \end{proof}

\begin{proof} [Proof of Lemma \ref{lem:varBRenyi} ]Adding and subtracting in the exponential $\int h dP $ and  $ \int hdQ$ resp we obtain the result:
$ \frac{1}{\alpha-1} \log \left( \int e^{(\alpha-1) h } dP \right)  - \frac{1}{\alpha} \log \left( \int e^{\alpha h } dQ \right) =  \frac{1}{\alpha-1} \log \left( \int e^{(\alpha-1) (h- \int h dP+\int h dP) } dP \right)  - \frac{1}{\alpha} \log \left( \int e^{\alpha (h-\int h dQ + \int h dQ) } dQ \right) = \int h dP - \int h dQ + \frac{1}{\alpha-1} \log \left( \int e^{(\alpha-1) ( h -\int h dP) } dP \right)  - \frac{1}{\alpha} \log \left( \int e^{\alpha ( h-   \int hdQ) } dQ \right) $

\end{proof}

\begin{proof}[ Proof of Lemma \ref{lem:limit0}]
Note that we have for $0<\alpha<1$, $\frac{1}{\alpha} D_{\alpha}( P || Q) = \frac{1}{1-\alpha}D_{1-\alpha}(Q || P)$ (See Proposition 2  in \cite{renyiPaper}). Taking limits we obtain $\lim_{\alpha \to 0 }\frac{1}{\alpha} D_{\alpha}( P || Q) =  D_1(Q || P )=  \mathsf{KL} ( Q || P).$
\end{proof}

\begin{proof} [Proof of Theorem \ref{theo:tailadapt} ] For $0 <\alpha<1$, we have for all $h \in \mathcal{H}$  : 

\begin{align}
\int h dP - \int h dQ &\leq \frac{1}{\alpha} D_{\alpha}(P || Q) +\frac{1}{1-\alpha} \log \left( \int e^{(\alpha-1) ( h -\int h dP) } dP \right) + \frac{1}{\alpha} \log \left( \int e^{\alpha ( h-   \int hdQ) } dQ \right)
\label{eq:renyi01}
\end{align}
Assuming $r$ is bounded $0<r<b$ then we have  $(r)_{\sharp} P - \mathbb{E}_{P} r$ and $(r)_{\sharp} Q - \mathbb{E}_{Q} r$ are sub-Gaussian with parameter $\sigma^2= \frac{b^2}{4}$. Hence we have for $\lambda \in \mathbb{R}$:
\[ \mathbb{E}_{P} e^{\lambda (r- \int r dP)} \leq \exp\left( \frac{\lambda^2\sigma^2_{P}}{2} \right) \text{ and }  \mathbb{E}_{Q} e^{\lambda (r- \int r dQ)} \leq \exp \left( \frac{\lambda^2\sigma^2_{Q}}{2}\right) ,  \]
Fix a finite $M>0$. For $0< \lambda< M$ and $P= \pi|x$ and $Q =\pi_{\mathrm{ref}}|x$, consider $h_{\lambda} = \lambda r$, thanks to subgaussianity and boundedness of $\lambda$, $h_{\lambda} \in \mathcal{H}$ for all $\lambda \in (0,M)$. Hence we have by Equation \eqref{eq:renyi01} for all $\lambda \in(0,M)$:

\[  \lambda \left(\int r dP- \int r dQ \right) \leq \frac{1}{\alpha} D_{\alpha}(P || Q) +\frac{1}{1-\alpha} \log \left( \int e^{\lambda(\alpha-1) ( r -\int r dP) } dP \right) + \frac{1}{\alpha} \log \left( \int e^{\lambda \alpha ( r-   \int rdQ) } dQ \right)
\]
we have by sub-Gaussianity: 
\begin{align*}
\frac{1}{1-\alpha} \log \left( \int e^{\lambda(\alpha-1) ( r -\int r dP) } dP \right) &  \leq \frac{1}{1-\alpha}  \frac{ \lambda^2 (1-\alpha)^2\sigma^2_{P}}{2} = \frac{\lambda^2(1-\alpha)\sigma^2_P }{2}\\
\frac{1}{\alpha} \log \left( \int e^{\lambda \alpha ( r-   \int rdQ) } dQ \right) & \leq \frac{1}{\alpha} \frac{\lambda^2\alpha^2 \sigma^2_Q}{2} =  \frac{\lambda^2\alpha \sigma^2_{Q}}{2} 
\end{align*}
It follows that for all $\lambda \in (0,M)$ 
\begin{align*}
  \lambda \left(\int r d\pi|x - \int r d\pi_{\mathrm{ref}}|x \right)& \leq \frac{1}{\alpha} D_{\alpha}(\pi|x ||  \pi_{\mathrm{ref}}|x ) + \frac{\lambda^2(1-\alpha)\sigma^2_P }{2} +  \frac{\lambda^2\alpha \sigma^2_Q}{2} \\
  &=  \frac{1}{\alpha} D_{\alpha}(\pi|x ||  \pi_{\mathrm{ref}}|x ) + \frac{\lambda^2 ((1-\alpha)\sigma^2_P + \alpha \sigma^2_Q ) }{2}
\end{align*}
Integrating over $x$ we obtain:
\begin{align*}
  \lambda \left(\int r d\pi - \int r d\pi_{\mathrm{ref}} \right) \leq \frac{1}{\alpha} D_{\alpha}(\pi  ||  \pi_{\mathrm{ref}} ) + \frac{\lambda^2 ((1-\alpha)\sigma^2_P + \alpha \sigma^2_Q )}{2}
  \end{align*}
Finally we have:
\begin{align*}
\int r d\pi - \int r d\pi_{\mathrm{ref}}  \leq \frac{1}{\lambda \alpha} D_{\alpha}(\pi  ||  \pi_{\mathrm{ref}} ) + \frac{\lambda ((1-\alpha)\sigma^2_P + \alpha \sigma^2_Q )}{2}
  \end{align*}
  minimizing over $\lambda \in(0,M) $:  we obtain $\lambda^*=  \sqrt{\frac{2 D_{\alpha} (\pi  ||  \pi_{\mathrm{ref}} ) }{ ((1-\alpha)\sigma^2_P + \alpha \sigma^2_Q ) \alpha}}$, $M$ is free of choice, choosing $M=2\lambda^*$, gives that $\lambda^*$ is the minimizer and hence we have for all $\alpha \in (0,1)$:
\begin{align*}
\int r d\pi - \int r d\pi_{\mathrm{ref}}  \leq \sqrt{\frac{2((1-\alpha)\sigma^2_P + \alpha \sigma^2_Q )  D_{\alpha} (\pi  ||  \pi_{\mathrm{ref}} )}{\alpha}}.
  \end{align*}  
  \end{proof}
  
  \section{Goodhart Laws}
  
  \begin{proof}[Proof of Proposition \ref{pro:TransferRL}]
We have by duality:
\[ \frac{1}{\beta}\log\left(\int e^{\beta r^*} d\pi_{\mathrm{ref}}\right) = \sup_{\nu} \int r^* d\nu - \frac{1}{\beta} \mathsf{KL} (\nu || \pi_{\mathrm{ref}})\]
hence for $\nu= \pi_{\beta,r}$ we have:
\[\frac{1}{\beta}\log\left(\int e^{\beta r^*} d\pi_{\mathrm{ref}}\right) \geq \int r^* d\pi_{\beta,r} - \frac{1}{\beta} \mathsf{KL} (\pi_{\beta,r} || \pi_{\mathrm{ref}} )
\]
Hence:
\begin{align*}
\int r^* d\pi_{\beta,r} \leq \frac{1}{\beta}\log\left(\int e^{\beta r^*} d\pi_{\mathrm{ref}}\right) + \frac{1}{\beta} \mathsf{KL} (\pi_{\beta,r} || \pi_{\mathrm{ref}} )
\end{align*}
On the other hand by optimality of $\pi_{\beta,r}$ we have:
\[ \mathsf{KL}\left(\pi_{\beta,r} || \pi_{\mathrm{ref}}\right)  = \beta \int r d\pi_{\beta,r} - \log\left( \int e^{\beta r} d\pi_{\mathrm{ref}}\right)\\
 \]
 Hence we have:
 \begin{align*}
 \int r^* d\pi_{\beta,r} &\leq \frac{1}{\beta}\log\left(\int e^{\beta r^*} d\pi_{\mathrm{ref}}\right) + \int r d\pi_{\beta,r} - \frac{1}{\beta} \log\left( \int e^{\beta r} d\pi_{\mathrm{ref}}\right) \leq \int r d\pi_{\beta,r} + \frac{1}{\beta} \log \left( \frac{\int e^{\beta r^*} d\pi_{\mathrm{ref}}}{\int e^{\beta r} d\pi_{\mathrm{ref}}}\right)
 \end{align*}
It follows that:
 \begin{align*}
 \int r^* d\pi_{\beta,r} - \int r^* d\pi_{\mathrm{ref}} & \leq \int r  d\pi_{\beta,r} -  \int r d\pi_{\mathrm{ref}}+  \frac{1}{\beta} \log \left( \frac{\int e^{\beta ( r^* -\int r^* d\pi_{\mathrm{ref}})} d\pi_{\mathrm{ref}}}{\int e^{\beta (r - \int r d\pi_{\mathrm{ref}})} d\pi_{\mathrm{ref}}}\right)
 \end{align*}
 \begin{align*}
\frac{\int e^{\beta ( r^* -\int r^* d\pi_{\mathrm{ref}})} d\pi_{\mathrm{ref}}}{\int e^{\beta (r - \int r d\pi_{\mathrm{ref}})} d\pi_{\mathrm{ref}}} &= \int e^{\beta(r^*-r - \left( \int r^* d\pi_{\mathrm{ref}} - \int r d\pi_{\mathrm{ref}} \right) }\frac{e^{\beta r} d\pi_{\mathrm{ref}}}{\int e^{\beta r } d\pi_{\mathrm{ref}} }\\
& = \int e^{\beta(r^*-r - \left( \int r^* d\pi_{\mathrm{ref}} - \int r d\pi_{\mathrm{ref}} \right) } d\pi_{\beta,r}
 \end{align*}
 Hence we have finally:
 \[  \int r^* d\pi_{\beta,r} - \int r^* d\pi_{\mathrm{ref}}  \leq \int r  d\pi_{\beta,r} -  \int r d\pi_{\mathrm{ref}}+  \frac{1}{\beta} \log \left(  \int e^{\beta(r^*-r - \left( \int r^* d\pi_{\mathrm{ref}} - \int r d\pi_{\mathrm{ref}} \right) } d\pi_{\beta,r} \right) \]
\[\int r^* d\pi_{\beta,r} - \int r^* d\pi_{\mathrm{ref}}  \leq \int r  d\pi_{\beta,r} -  \int r d\pi_{\mathrm{ref}}-  \frac{1}{\beta} \log \left(  \int e^{\beta(r-r^* - \left( \int r d\pi_{\mathrm{ref}} - \int r^* d\pi_{\mathrm{ref}} \right) } d\pi_{\beta,r^*} \right)  \]

The proof follows from using the subgaussianity of $r_{\sharp} \pi_{\mathrm{ref}}$ and the assumption on the soft max. 
 \end{proof}

  \begin{proof} [Proof of Proposition \ref{pro:goohardtBestofn}]
  \begin{align*}
\mathbb{E}_{\pi} (r^* -r) - \mathbb{E}_{\pi_{\mathrm{ref}}} (r^*-r) \leq 2 ||r-r^*||_{\infty} \mathsf{TV}(\pi,\pi_{\mathrm{ref}})
\end{align*}

For $\pi^{(n)}_{r,\mathrm{ref}}$, we have: 
\[ \mathbb{E}_{\pi^{(n)}_{r,\mathrm{ref}}} (r^*) - \mathbb{E}_{\pi_{\mathrm{ref}}} (r^*) \leq \mathbb{E}_{\pi^{(n)}_{r,\mathrm{ref}}} (r) - \mathbb{E}_{\pi_{\mathrm{ref}}} (r) +2 ||r-r^*||_{\infty} \mathsf{TV}(\pi^{(n)}_{r,\mathrm{ref}} ,\pi_{\mathrm{ref}}) \]
and 
\[ \mathbb{E}_{\pi^{(n)}_{r,\mathrm{ref}}} (r^*) - \mathbb{E}_{\pi_{\mathrm{ref}}} (r^*) \geq \mathbb{E}_{\pi^{(n)}_{r,\mathrm{ref}}} (r) - \mathbb{E}_{\pi_{\mathrm{ref}}} (r) -2 ||r-r^*||_{\infty} \mathsf{TV}(\pi^{(n)}_{r,\mathrm{ref}} ,\pi_{\mathrm{ref}}) \]
By the data processing inequality we have:
$  \mathsf{TV}(\pi^{(n)}_{r,\mathrm{ref}} ,\pi_{\mathrm{ref}}) \leq  \mathsf{TV}(R^{(n)}_{r,\mathrm{ref}}, R) =( \frac{1}{n})^{\frac{1}{n-1}}-  (\frac{1}{n})^{\frac{n}{n-1}}$
If $r$ has subguassian tails under $\pi_{\mathrm{ref}}$ than we have:
\[ \mathbb{E}_{\pi^{(n)}_{r,\mathrm{ref}}} (r^*) - \mathbb{E}_{\pi_{\mathrm{ref}}} (r^*) \leq \sqrt{2\sigma^2 \left( \log(n) - \frac{n-1}{n}\right)} +2 ||r-r^*||_{\infty}  \left( (\frac{1}{n})^{\frac{1}{n-1}}-  (\frac{1}{n})^{\frac{n}{n-1}} \right)\]
\[ \mathbb{E}_{\pi^{(n)}_{r,\mathrm{ref}}} (r^*) - \mathbb{E}_{\pi_{\mathrm{ref}}} (r^*) \leq \sqrt{2\sigma^2 \left( \log(n) - \frac{n-1}{n}\right)} +2 \inf_{r\in \mathcal{H}} ||r-r^*||_{\infty}  \left( (\frac{1}{n})^{\frac{1}{n-1}}-  (\frac{1}{n})^{\frac{n}{n-1}} \right).\]
\end{proof}
\end{document}